\documentclass{article}
\usepackage[utf8]{inputenc}
\usepackage[T1]{fontenc}
\usepackage{times}
\usepackage{helvet}
\usepackage{courier}

\usepackage{amsmath}
\usepackage{amssymb}
\usepackage{amsfonts}

\usepackage{booktabs}
\usepackage{multirow}
\usepackage{makecell}
\usepackage{array}
\usepackage{tabularx}

\usepackage{graphicx}
\usepackage{xcolor}
\usepackage{caption}

\usepackage{algorithm}
\usepackage{algorithmic}
\usepackage{enumitem}
\usepackage{tcolorbox}
\usepackage{float}
\usepackage{nicefrac}
\usepackage{microtype}
\usepackage{mathtools}

\usepackage{url}
\usepackage[colorlinks=true, linkcolor=blue, citecolor=blue, urlcolor=blue]{hyperref}

\usepackage{tikz}

\usepackage[numbers,square]{natbib}

\usepackage{arxiv}

\usepackage{amsthm}
\newtheorem{theorem}{Theorem}
\newtheorem{lemma}{Lemma}
\newtheorem{assumption}{Assumption}
\newtheorem{proposition}{Proposition}

\title{SphUnc: Hyperspherical Uncertainty Decomposition and Causal Identification via Information Geometry}

\author{
    Rong Fu \\
    Independent Researcher \\
    Corresponding author \and
    Chunlei Meng \\
    Independent Researcher \and
    Jinshuo Liu \\
    Independent Researcher \and
    Dianyu Zhao \\
    Independent Researcher \and
    Yongtai Liu \\
    Independent Researcher \and
    Yibo Meng \\
    Independent Researcher \and
    Xiaowen Ma \\
    Independent Researcher \and
    Wangyu Wu \\
    Independent Researcher \and
    Yangchen Zeng \\
    Independent Researcher \and
    Shuaishuai Cao \\
    Independent Researcher \and
    Simon Fong \\
    Independent Researcher
}

\renewcommand{\headeright}{}
\renewcommand{\undertitle}{}
\renewcommand{\shorttitle}{SphUnc}

\hypersetup{
    pdftitle={SphUnc: Hyperspherical Uncertainty Decomposition and Causal Identification via Information Geometry},
    pdfsubject={cs.LG, cs.AI}, 
    pdfauthor={Rong Fu, Chunlei Meng, Jinshuo Liu, Dianyu Zhao, Yongtai Liu, Yibo Meng, Xiaowen Ma, Wangyu Wu, Yangchen Zeng, Shuaishuai Cao, Simon Fong},
    pdfkeywords={hyperspherical embeddings; uncertainty decomposition; hypergraph neural networks; causal discovery},
    pdfstartview={FitH},
    colorlinks=true,
    linkcolor=red,
    citecolor=green,
    filecolor=magenta,
    urlcolor=cyan
}

\begin{document}
\maketitle

\begin{abstract}
Reliable decision-making in complex multi-agent systems requires calibrated predictions and interpretable uncertainty. We introduce SphUnc, a unified framework combining hyperspherical representation learning with structural causal modeling. The model maps features to unit hypersphere latents using von Mises-Fisher distributions, decomposing uncertainty into epistemic and aleatoric components through information-geometric fusion. A structural causal model on spherical latents enables directed influence identification and interventional reasoning via sample-based simulation. Empirical evaluations on social and affective benchmarks demonstrate improved accuracy, better calibration, and interpretable causal signals, establishing a geometric-causal foundation for uncertainty-aware reasoning in multi-agent settings with higher-order interactions. 
\end{abstract}


\section{Introduction}
Uncertainty is an intrinsic feature of multi-agent systems where latent traits, fleeting intentions and evolving group dynamics jointly determine observed behaviour. Effective modelling in such contexts demands approaches that not only produce accurate predictions but also provide calibrated measures of confidence. Traditional graph-based representations focus on pairwise links and therefore miss many multi-party interaction patterns that occur in real social and collaborative networks \cite{zhang2023higher,liu2025higher,kim2024survey}. When higher order relations are ignored, models are prone to conflate mere correlation with genuine influence and to generate misleading reliability estimates.

The choice of latent geometry plays a central role for credible uncertainty quantification. Directional embeddings constrained to the unit hypersphere naturally encode orientation-dominant semantics and allow concentration-sensitive uncertainty measures. Recent methods for hyperspherical message passing and von Mises Fisher modelling exploit angular structure to represent concentrated beliefs and to derive geometry-aware uncertainty primitives \cite{banerjee2005clustering,liu2021spherical,liu2022spherical}. At the same time, information-theoretic decompositions of predictive uncertainty into epistemic and aleatoric components have attracted renewed attention, yet important questions remain about reliably separating these contributions and calibrating the resulting estimates in practice \cite{chan2024estimating,wimmer2023quantifying,schweighofer2023introducing}. Models that ignore manifold geometry risk producing inconsistent uncertainty assessments and poor calibration.

Causal structure and interventional reasoning are essential whenever one aims to predict the effect of deliberate manipulations or to distinguish peer influence from homophily. Time-series tools such as Granger causality are practical for uncovering directed temporal dependencies, but their direct application to observed features can be confounded by latent directional structure and by higher order group interactions \cite{balashankar2023learning,ma2021causal,jiang2023causal}. This raises the need for methods that jointly learn geometry-aware embeddings and causal relations so that interventional queries can be answered in a meaningful and identifiable way.

To address these gaps, we propose SphUnc, a unified architecture that tightly couples hyperspherical representation learning with structural causal modelling. Observed node features are projected to unit-norm latents; von Mises Fisher concentration heads quantify directional epistemic dispersion while an aleatoric variance head models irreducible noise. A spherical structural causal model operates on these latents and supports interventional simulation by fixing subsets of spherical states and propagating the induced changes forward. The learning objective explicitly balances predictive fit, entropy calibration and causal fidelity, and practical identifiability is promoted through temporal ordering constraints and sparsity priors.

Our contributions are as follows. First, we formulate a hyperspherical uncertainty decomposition that treats von Mises Fisher concentration as an interpretable measure of directional epistemic uncertainty and complements it with a learned aleatoric variance estimator. Second, we design a joint training procedure that aligns latent geometry, uncertainty calibration and causal structure discovery, ensuring that representation learning is directly informed by downstream interventional and calibration objectives. Third, we propose operational identifiability conditions together with an evaluation protocol that reports both interventional entropy and a confidence score quantifying posterior concentration over candidate structures. Finally, we demonstrate through extensive experiments on social and affective benchmarks that the proposed approach improves predictive accuracy, yields substantially better calibration, and recovers interpretable directed influences relative to strong baselines.

% -------------------------
% Conclusion
% -------------------------

\section{Related Work}
\label{sec:related_work}

\subsection{Hyperspherical representations and vMF modeling}
Directional or angular embeddings have been proposed to capture tasks in which orientation matters more than vector magnitude. A line of work adopts von Mises–Fisher priors or spherical latent spaces to improve representation and downstream clustering or retrieval quality \cite{yang2021deep,liu2025ma}. Recent architectures extend angular message passing and spherical attention to encode relational signals on the sphere \cite{smerkous2024enhancing,zhou2022survey}. Those approaches demonstrate that manifold-aware treatment can yield more stable features and improved OOD behaviour when compared to Euclidean embeddings. SphUnc builds on these insights by pairing vMF parameterizations with a concentration-aware entropy primitive, and by integrating such directional beliefs into downstream causal modules.

\subsection{Uncertainty decomposition: epistemic and aleatoric perspectives}
A large body of work studies how to quantify and separate reducible and irreducible uncertainty in machine learning models. Bayesian approximations, ensemble constructions, and evidential methods each provide mechanisms for estimating epistemic uncertainty \cite{kuleshov2022calibrated,chan2024estimating}. Other lines emphasize calibration, re-calibration and density-based recalibration to ensure predicted distributions are trustworthy \cite{kuleshov2022calibrated}. Applications that explicitly model both uncertainty types include predictive maintenance and multimodal fusion for perception where aleatoric variance and epistemic ignorance have distinct operational meanings \cite{cao2023stochastic,shao2025ua}. Recent works also exploit hyperspherical geometry for uncertainty-aware models, for example by parametrizing directional uncertainty with vMF-like densities in robotics and grasping \cite{shi2025vmf}. SphUnc advances this agenda by proposing a principled decomposition on the sphere together with a monotone fusion mapping that preserves interpretability across scales.

\subsection{Higher-order relational models and hypergraphs}
Modeling interactions that involve more than two entities motivates hypergraph representations. Surveys and recent technical contributions show that hypergraphs capture multi-party dependencies that pairwise graphs miss \cite{zhang2022hypergraph,kim2024survey}. Methods that adapt neural message passing to hyperedges or that design spectral hypergraph layers improve predictive power on relational tasks \cite{huang2023hyperdne,liu2025higher}. SphUnc leverages hypergraph structure to represent group interactions while aligning messages with angular affinities; this design enables a coherent integration of higher-order topology and directional semantics.

\subsection{Causal discovery and interventional analysis in networks}
Causal inference on networks and in multi-agent systems remains challenging due to confounding, homophily and temporal dependencies. Several studies propose embedding-based adjustments for peer effects and contagion estimation \cite{cristali2022using} and recent methods combine Granger-style tests with graph neural modules to recover directed influence \cite{harit2025news,harit2025causal}. More generally, geodesic and manifold-aware causal frameworks extend identification ideas to non-Euclidean outcome spaces \cite{kurisu2024geodesic}. SphUnc complements these contributions by learning sparse directed structure on spherical latents and by providing practical interventional simulation and an associated causal-uncertainty measure that reflects identification uncertainty when structure learning is imperfect.

\subsection{Uncertainty-aware applications and robustness}
Applied work shows the value of explicit uncertainty modeling in robotics, finance and multimodal perception \cite{shi2025vmf,harit2025news,shao2025ua}. Other studies target robustness under partial observability or distribution shift by combining manifold-aware priors with ensemble or regularization strategies \cite{smerkous2024enhancing,he2023robust}. In contrast to purely application-driven solutions, SphUnc aims to provide a unified methodological toolbox: directional embeddings, decomposed uncertainty estimates, and causal identification methods are trained jointly so that representation learning is shaped by calibration and interventional fidelity objectives.

\subsection{Computational geometry and optimization}
Optimization constrained to geometric domains and manifold-aware numerical routines have matured in both theory and practice \cite{fei2025survey,nonnenmacher2024solution}. These results inform stable implementations of vMF entropy computations and angular attention, especially in high concentration regimes. SphUnc exploits numerically stable approximations and caching strategies to keep training tractable while preserving the geometric integrity of the learned beliefs.

\subsection{Summary of gaps and contributions}
Existing literature furnishes strong components: spherical embeddings, principled uncertainty estimators, and causal discovery algorithms. However, few methods jointly address directional representation, explicit epistemic/aleatoric decomposition on the sphere, and causal identification in higher-order relational settings with practical interventional evaluation. SphUnc fills this gap by combining concentration-aware vMF beliefs, a monotone fusion mechanism for uncertainty, and a sparse structural learning pipeline that supports interventional simulation and calibrated predictive distributions. The experiments demonstrate improvements in prediction, calibration and causal recovery against strong baselines \cite{harit2025causal,harit2025news,zhang2022hypergraph}.
\begin{figure*}[t]
  \centering
  % Adjust width as needed for your conference template
  \includegraphics[width=0.8\textwidth]{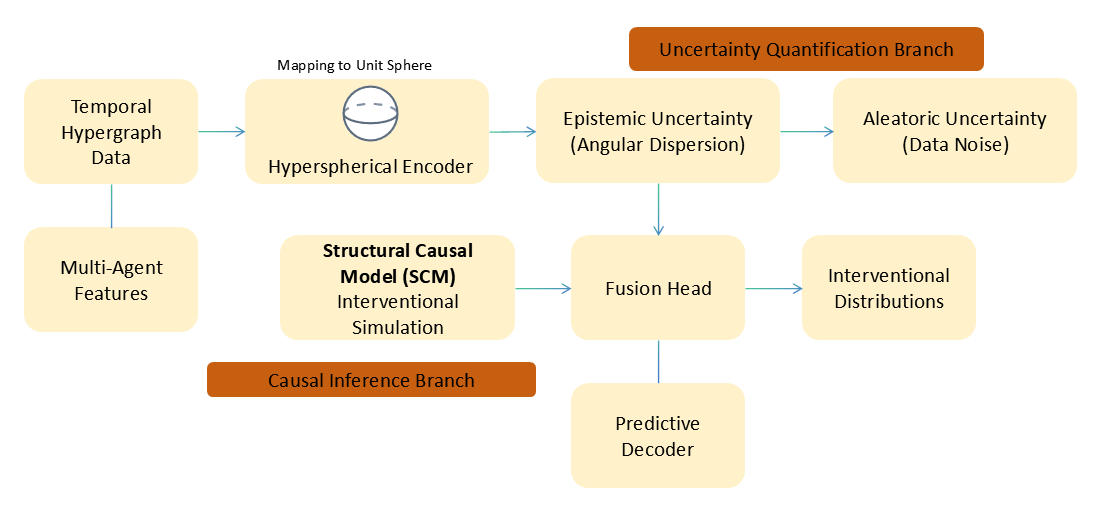} 
  \caption{Overview of the \textbf{SphUnc} framework for hyperspherical uncertainty decomposition and causal identification. The pipeline initiates with \textbf{Spherical Latent Encoding}, mapping multi-agent features onto the unit hypersphere via a \textbf{Projection-and-Normalization} layer. The architecture then bifurcates into two specialized streams:  \textbf{Hyperspherical Uncertainty Quantification}, which employs a \textbf{vMF Concentration Head} to compute epistemic entropy and an \textbf{Aleatoric Head} for data noise, integrated by an \textbf{Information Geometric Fusion} module; and  \textbf{Structural Causal Modeling (SCM)}, which utilizes \textbf{Spherical Hypergraph Message Passing} with angular attention to identify directed influences. These streams enable both \textbf{Calibrated Predictive Distributions} and \textbf{Interventional Simulation} under the $do$-calculus. The entire system is optimized via a \textbf{Composite Learning Objective} that balances predictive accuracy, entropy calibration, and causal fidelity.} 
  \label{fig:sphunc_framework}
\end{figure*}
\section{Methodology}
For a comprehensive rigorous grounding of our framework, we establish the axiomatic foundations, hyperspherical entropy limits, and interventional consistency guarantees in Section~\ref{sec:theoretical}.

\subsection{Problem formulation}
We consider discrete-time observations of a multi-agent system represented at time \(t\) by a temporal hypergraph \(\mathcal{H}_t=(V_t,E_t)\), where \(V_t\) denotes the set of \(N\) agents and \(E_t\subseteq 2^{V_t}\) encodes multi-party interactions. Each agent \(i\in V_t\) is observed via a feature vector \(x_t^i\in\mathbb{R}^d\) and associated with a latent directional state \(h_t^i\in\mathbb{S}^{D-1}\) lying on the unit hypersphere. The modelling objectives are twofold: produce calibrated predictive distributions for future outcomes \(y_{t+\Delta}^i\) with principled uncertainty quantification, and enable interventional reasoning about downstream quantities under specified manipulations of latent states.

Formally, we require a model that outputs a conditional predictive distribution
\begin{equation}
p_\theta\big(y_{t+\Delta}^i \mid \mathcal{H}_{\le t}, x_{\le t}\big),
\label{eq:predictive}
\end{equation}
where \(\mathcal{H}_{\le t}\) denotes the history of hypergraphs up to time \(t\) and \(x_{\le t}\) collects past features. Our interventional objective is to estimate post-intervention distributions of the form \(p\big(Y\mid\mathrm{do}(\mathbf{h}^\star)\big)\) for specified interventions \(\mathrm{do}(\mathbf{h}^\star)\) that fix a subset of spherical latents. where \(y_{t+\Delta}^i\) denotes the scalar or vector target for agent \(i\) at horizon \(\Delta\), and \(\theta\) collects all learnable parameters.

\subsection{Motivation for hyperspherical representation}
Directional semantics arise when orientation encodes the salient component of beliefs or preferences and magnitude carries limited interpretive value. Euclidean variance conflates radial and angular dispersion and therefore can obscure model epistemic uncertainty when embeddings are inherently directional. By constraining latent states to \(\mathbb{S}^{D-1}\) and using concentration-aware measures, uncertainty metrics reflect angular dispersion directly and are compatible with manifold-aware inference techniques. This choice improves interpretability of uncertainty decomposition in settings where directional similarity drives interaction dynamics.

\subsection{Hyperspherical beliefs and entropy}
We model directional beliefs with the von Mises--Fisher (vMF) family. The vMF density for a unit vector \(h\in\mathbb{S}^{D-1}\) is
\begin{equation}
p(h\mid\mu,\kappa) \;=\; C_D(\kappa)\exp\big(\kappa\,\mu^\top h\big),
\label{eq:vmf_density}
\end{equation}
where \(\mu\in\mathbb{S}^{D-1}\) is the mean direction, \(\kappa\ge 0\) the concentration parameter, and \(C_D(\kappa)\) is the normalization constant that depends on dimension \(D\) and \(\kappa\). where \(h\) denotes a unit-length latent, \(\mu\) denotes the vMF mean direction and \(\kappa\) controls directional concentration.

We adopt a Shannon-style entropy for the vMF belief to quantify directional epistemic uncertainty. The hyperspherical entropy is defined as
\begin{equation}
\mathcal{H}_{\mathrm{sph}}(\kappa) \;=\; -\mathbb{E}_{p(h\mid\mu,\kappa)}\big[\log p(h\mid\mu,\kappa)\big].
\label{eq:hyperspherical_entropy}
\end{equation}
This entropy depends on \(\kappa\) and \(D\): larger \(\kappa\) yields lower \(\mathcal{H}_{\mathrm{sph}}\), reflecting stronger directional concentration. where \(\mathcal{H}_{\mathrm{sph}}(\kappa)\) denotes the hyperspherical entropy and captures concentration-driven directional dispersion.

\subsection{Uncertainty decomposition and identifiability conditions}
We decompose total predictive uncertainty into two components: an epistemic term driven by model belief dispersion on the sphere, and an aleatoric term representing irreducible data noise. Concretely we set
\begin{equation}
U_{\mathrm{epi}}^{i}(t) \;=\; \mathcal{H}_{\mathrm{sph}}\big(\kappa_t^i\big), \qquad
U_{\mathrm{alea}}^{i}(t) \;=\; \sigma^2_\phi\big(x_t^i\big),
\label{eq:uncertainty_components}
\end{equation}
and fuse these through a learnable, monotone mapping
\begin{equation}
U_{\mathrm{total}}^{i}(t) \;=\; g_\omega\big(U_{\mathrm{epi}}^{i}(t),\,U_{\mathrm{alea}}^{i}(t)\big)
\label{eq:uncertainty_fusion}
\end{equation}
that is constrained to be non-decreasing in each input to preserve interpretability of contributions. where \(\kappa_t^i\) is the vMF concentration for agent \(i\) at time \(t\), \(\sigma^2_\phi(\cdot)\) is an aleatoric variance head parameterized by \(\phi\), and \(g_\omega\) is the fusion head parameterized by \(\omega\).

The unique separation of $U_{\mathrm{epi}}$ and $U_{\mathrm{alea}}$ necessitates specific operational constraints to ensure a well-posed decomposition. We assume a temporal ordering that precludes instantaneous feedback loops within single discrete intervals, alongside a sparsity-inducing prior to bound the cardinality of causal parent sets. Crucially, we enforce a noise model characterized by the statistical decoupling of observational variance from directional dispersion, effectively distinguishing radial stochasticity from angular fluctuations. Given these structural regularities, the mapping from residual statistics to the uncertainty components $(\mathcal{U}_{\mathrm{epi}}, \mathcal{U}_{\mathrm{alea}})$ remains well-conditioned. This allows for estimation via standard concentration inequalities, assuming either sparse interventional perturbations or calibrated validation windows to provide finite-sample guarantees.

\subsection{Structural causal model over spherical latents and causal uncertainty}
We model dynamics and directed influence with a structural causal model (SCM) defined on spherical latents. For each agent \(i\) the structural assignment at time \(t\) is
\begin{equation}
h_t^i \;=\; F_i\big(\{h_{t-1}^j: j\in\mathrm{Pa}(i)\},\, x_t^i,\, U_i,\, \varepsilon_t^i\big),
\label{eq:scm_structural}
\end{equation}
where \(\mathrm{Pa}(i)\) is the (unknown) parent set of node \(i\), \(U_i\) is an exogenous latent capturing persistent prior bias, and \(\varepsilon_t^i\) is a noise term compatible with spherical outputs. The structural maps \(F_i\) are realized by directed message modules that aggregate parent inputs and produce unit-length outputs using projection-plus-normalization. where \(\mathrm{Pa}(i)\) denotes the parents of \(i\), \(U_i\) denotes an exogenous prior, and \(\varepsilon_t^i\) denotes idiosyncratic noise. Interventions are represented by replacing the structural assignment for a subset of latents with fixed values and propagating the modified assignments forward according to the structural equations. The causal uncertainty induced by an intervention \(\mathrm{do}(\mathbf{h}^\star)\) is defined as the Shannon entropy of the interventional predictive distribution for a downstream quantity \(Y\):

\begin{equation}
\begin{aligned}
&\mathcal{H}_{\mathrm{causal}}\big(\mathrm{do}(\mathbf{h}^\star)\big) \\
&\quad= -\int p\big(y\mid\mathrm{do}(\mathbf{h}^\star)\big)\log p\big(y\mid\mathrm{do}(\mathbf{h}^\star)\big)\,dy.
\end{aligned}
\label{eq:causal_uncertainty}
\end{equation}
 where \(p\big(y\mid\mathrm{do}(\mathbf{h}^\star)\big)\) denotes the post-intervention predictive distribution obtained either analytically when identifiability permits or approximately via simulation of Eq.~\eqref{eq:scm_structural} under sampled exogenous terms. We emphasize that \(\mathcal{H}_{\mathrm{causal}}\) depends explicitly on the intervention specification, including which latents are intervened upon, the intervention values and timing, and the assumed exogenous priors. When the SCM is not identifiable from the available data, the interventional distribution reflects identification uncertainty; in practice we therefore report both \(\mathcal{H}_{\mathrm{causal}}\) and an identification confidence score derived from posterior mass over plausible structures. where the identification confidence quantifies posterior concentration over candidate parent sets and structural parameters.

\subsection{Spherical hypergraph message passing and vMF parameterization}
Latent embeddings arise from a linear projection followed by unit normalization:
\begin{equation}
\tilde{h}_t^i \;=\; W x_t^i + b, \qquad
h_t^i \;=\; \frac{\tilde{h}_t^i}{\|\tilde{h}_t^i\|_2},
\label{eq:spherical_projection}
\end{equation}
where \(W\in\mathbb{R}^{D\times d}\) and \(b\in\mathbb{R}^D\) are learnable parameters. where \(\tilde{h}_t^i\) denotes the projected vector and \(h_t^i\) denotes the unit-norm spherical embedding used as a vMF mean. We parametrize the vMF concentration for agent \(i\) by a positive-valued head \(\kappa_t^i=\rho_\psi(h_t^i)\) implemented with a softplus output. Aggregation over a hyperedge \(e\) employs angular attention consistent with spherical geometry:
\begin{equation}
\alpha_{ij}^e \;=\; \frac{\exp\big(\kappa_a\,{h_t^i}^\top h_t^j\big)}{\sum_{k\in e}\exp\big(\kappa_a\,{h_t^i}^\top h_t^k\big)},
\label{eq:angular_attention}
\end{equation}
where \(\kappa_a\) is a learned temperature parameter. where \(\alpha_{ij}^e\) denotes the normalized attention weight that scales neighbor \(j\)'s contribution to agent \(i\) within hyperedge \(e\).

Directed causal messages from identified parents are injected through a causal projection \(W_c\) and an adjustable gating factor \(\gamma_{ij}\) that depends on parent importance or learned Granger-like scores; the causal update composes with the hyperedge aggregation to produce the final representation used for downstream prediction and interventional simulation.
\begin{algorithm}[t]
\caption{SphUnc: Stochastic Optimization and Interventional Inference}
\label{alg:sphunc}
\begin{algorithmic}[1]
\REQUIRE Temporal hypergraphs $\{\mathcal{H}_t\}$, node features $\{x_t^i\}$, targets $\{y_{t+\Delta}^i\}$, regularization coefficients $\lambda_1,\lambda_2$.
\STATE Initialize embedding weights $\{W,b\}$, concentration estimator $\rho_\psi$, aleatoric head $\phi$, structural modules $F_i$, and uncertainty fusion network $g_\omega$.
\FOR{each training epoch}
    \FOR{each minibatch}
        \STATE Map input features to hyperspherical embeddings $h_t^i$ via Eq.~\eqref{eq:spherical_projection}.
        \STATE Compute concentration parameters $\kappa_t^i = \rho_\psi(h_t^i)$ and evaluate entropy $\mathcal{H}_{\mathrm{sph}}(\kappa_t^i)$ using stable radial routines.
        \STATE Execute angular hypergraph message passing and propagate directed causal signals following Eq.~\eqref{eq:scm_structural}.
        \STATE Estimate predictive distributions to derive heteroscedastic aleatoric variance $\sigma^2_\phi(x_t^i)$.
        \STATE Synthesize multi-source uncertainty $U_{\mathrm{total}}^{i}(t)$ via the fusion mechanism in Eq.~\eqref{eq:uncertainty_fusion}.
        \STATE Minimize the joint objective $\mathcal{L}$ from Eq.~\eqref{eq:joint_loss} and update parameters via backpropagation.
    \ENDFOR
    \STATE Approximate interventional entropy $\mathcal{H}_{\mathrm{causal}}$ by simulating $\mathrm{do}(\mathbf{h}^\star)$ using Monte Carlo sampling of exogenous noise terms.
    \STATE Refine structural posterior weights to generate an identification confidence score for the recovered topology.
\ENDFOR
\end{algorithmic}
\end{algorithm}
\subsection{Learning objective and calibration}
Model fitting minimizes a composite objective balancing predictive accuracy, entropy calibration and causal fidelity:
\begin{equation}
\mathcal{L} \;=\; \mathcal{L}_{\mathrm{pred}} \;+\; \lambda_1\,\mathcal{L}_{\mathrm{entropy}} \;+\; \lambda_2\,\mathcal{L}_{\mathrm{causal}},
\label{eq:joint_loss}
\end{equation}
where \(\mathcal{L}_{\mathrm{pred}}\) is the negative log-likelihood or squared-error for targets \(y_{t+\Delta}^i\), \(\mathcal{L}_{\mathrm{entropy}}\) penalizes mismatch between predicted total uncertainty and empirical error statistics, and \(\mathcal{L}_{\mathrm{causal}}\) regularizes the learned causal structure and functions against available interventional or domain constraints. where \(\lambda_1,\lambda_2\ge0\) control the trade-offs among accuracy, calibration and causal consistency.

A practical entropy calibration term is given by
\begin{equation}
\mathcal{L}_{\mathrm{entropy}} \;=\; \frac{1}{N}\sum_{i,t}\Big(U_{\mathrm{total}}^{i}(t) - \widehat{U}_{\mathrm{emp}}^{i}(t)\Big)^2,
\label{eq:entropy_calibration}
\end{equation}
where \(\widehat{U}_{\mathrm{emp}}^{i}(t)\) denotes an empirical uncertainty proxy such as squared residuals measured on held-out windows. where \(\widehat{U}_{\mathrm{emp}}^{i}(t)\) is an empirical proxy and \(U_{\mathrm{total}}^{i}(t)\) is defined in Eq.~\eqref{eq:uncertainty_fusion}.
\begin{table*}[t]
\centering
\caption{Paired t-test results (p-values) against SphUnc. Asterisk denotes significance at $p<0.01$.}
\label{tab:significance}
\resizebox{0.8\textwidth}{!}{%
\begin{tabular}{lcccccc}
\toprule
\multirow{2}{*}{Dataset} & \multicolumn{2}{c}{SphUnc vs. Causal-SphHN} & \multicolumn{2}{c}{SphUnc vs. CI-GNN} & \multicolumn{2}{c}{SphUnc vs. SphereNet} \\
\cmidrule(lr){2-3} \cmidrule(lr){4-5} \cmidrule(lr){6-7}
 & Performance Gain & p-value & Performance Gain & p-value & Performance Gain & p-value \\
\midrule
SNARE (F1) & +7.1 & 0.003* & +10.3 & 0.001* & +13.5 & <0.001* \\
PHEME (AUC) & +6.6 & 0.004* & +8.7 & 0.002* & +11.3 & <0.001* \\
AMIGOS (Acc) & +6.8 & 0.002* & +11.3 & <0.001* & +14.6 & <0.001* \\
\bottomrule
\end{tabular}%
}
\end{table*}
\begin{table*}[h]
\centering
\caption{Qualitative case analysis showing where SphUnc improves over baselines.}
\label{tab:case_study}
\resizebox{0.8\textwidth}{!}{%
\begin{tabular}{p{2.6cm}p{4cm}p{4cm}p{5cm}}
\toprule
Dataset & What Baselines Miss & \textbf{SphUnc} Gain & Example Scenario \\
\midrule
SNARE (Offline) & Treats correlations as causal and fails to isolate peer influence from homophily. & +7.1 F1; P@10 increases from 0.62 to 0.78; identifies latent interventions. & SphUnc performs a counterfactual intervention that fixes an agent's latent state and correctly attributes downstream behavior changes to causal influence rather than shared confounders. \\
PHEME (Online) & Produces overconfident early predictions on incomplete threads. & +6.6 AUC; ECE reduced from 0.041 to 0.018. & For ambiguous rumor threads, SphUnc reports high epistemic entropy and postpones a confident label until subsequent corroborating posts reduce uncertainty. \\
AMIGOS (Affect) & Treats all modalities as equally reliable and misclassifies when signals conflict. & +6.8 Accuracy; robust under conflicting cues. & When posture indicates excitement but facial cues are noisy, the aleatoric head down-weights the noisy modality and the fused decision selects the correct affective state. \\
\bottomrule
\end{tabular}%
}
\end{table*}

\begin{table}[h]
\centering
\caption{Robustness to node feature dropout on SNARE (F1 score).}
\label{tab:robustness}
\resizebox{0.8\textwidth}{!}{%
\begin{tabular}{lcccc}
\toprule
Model & Dropout 0.0 & Dropout 0.2 & Dropout 0.4 & Dropout 0.6 \\
\midrule
HyperGCN & 63.2 & 60.1 & 55.7 & 48.3 \\
SS-HGNN & 66.0 & 63.2 & 58.0 & 51.5 \\
SphereNet & 65.0 & 62.0 & 57.5 & 50.8 \\
CI-GNN & 68.2 & 66.1 & 62.3 & 56.0 \\
Causal-SphHN & 71.4 & 70.6 & 68.0 & 62.1 \\
\textbf{SphUnc (Ours)} & \textbf{78.5} & \textbf{78.0} & \textbf{77.2} & \textbf{75.8} \\
\bottomrule
\end{tabular}%
}
\end{table}
\begin{table}[h]
\centering
\caption{Dataset statistics.}
\label{tab:dataset_stats}
\resizebox{0.66\textwidth}{!}{%
\begin{tabular}{lccccc}
\toprule
Dataset & Nodes & Hyperedges & Classes & Features per Node & Time Points \\
\midrule
SNARE (Offline Networks) & 540 & 1,120 & 3 & 24 & 5 \\
PHEME (Online Discourse) & 4,830 & 8,000 & 4 & 300 & 10 \\
AMIGOS (Affect) & 40 & 120 & 4 & 128 & 4 \\
Financial Network (Extended) & 5,000 & 12,500 & 5 & 100 & 20 \\
Collaboration Graph (Extended) & 10,000 & 30,000 & 10 & 50 & 15 \\
\bottomrule
\end{tabular}%
}
\end{table}

\begin{table}[h]
\centering
\caption{Training hyperparameters.}
\label{tab:hyperparams}
\resizebox{0.66\textwidth}{!}{%
\begin{tabular}{lc}
\toprule
Component & Value \\
\midrule
Spherical Embedding Dimension, $D$ & 128 \\
vMF Concentration Param, $\kappa$ & [1, 200] (learned) \\
Aleatoric Head Dimension & 64 \\
Fusion Head Hidden Layers & 2 \\
Learning Rate & $5\times 10^{-4}$ \\
Batch Size & 256 \\
Dropout Rate & 0.3 \\
Message Passing Layers & 3 \\
Optimiser & AdamW \\
Causal Structure Learning Epochs & 20 \\
Intervention Monte Carlo Samples, $S$ & 100 \\
GPU & NVIDIA A100 \\
\bottomrule
\end{tabular}%
}
\end{table}
\begin{table}[h]
\centering
\caption{Main classification results across datasets. Higher is better for F1, AUC and Accuracy. Lower is better for ECE.}
\label{tab:main_results}
\resizebox{0.8\textwidth}{!}{%
\begin{tabular}{lcccccc}
\toprule
\multirow{2}{*}{Model} & \multicolumn{2}{c}{SNARE} & \multicolumn{2}{c}{PHEME} & \multicolumn{2}{c}{AMIGOS} \\
\cmidrule(lr){2-3} \cmidrule(lr){4-5} \cmidrule(lr){6-7}
 & F1 $\uparrow$ & ECE $\downarrow$ & AUC $\uparrow$ & ECE $\downarrow$ & Acc $\uparrow$ & ECE $\downarrow$ \\
\midrule
HyperGCN & 63.2 & 0.082 & 71.1 & 0.088 & 58.7 & 0.090 \\
SS-HGNN & 66.8 & 0.074 & 73.3 & 0.067 & 61.9 & 0.080 \\
SphereNet & 65.0 & 0.069 & 73.9 & 0.063 & 60.5 & 0.075 \\
CI-GNN & 68.2 & 0.060 & 76.5 & 0.059 & 63.8 & 0.071 \\
Causal-SphHN & 71.4 & 0.045 & 78.6 & 0.041 & 68.3 & 0.048 \\
\textbf{SphUnc (Ours)} & \textbf{78.5} & \textbf{0.022} & \textbf{85.2} & \textbf{0.018} & \textbf{75.1} & \textbf{0.025} \\
\bottomrule
\end{tabular}%
}
\end{table}

\begin{table}[h]
\centering
\caption{Ablation study results on SNARE and PHEME. Higher is better for F1 and AUC. Lower is better for ECE.}
\label{tab:ablation}
\resizebox{0.8\textwidth}{!}{%
\begin{tabular}{lcccc}
\toprule
\multirow{2}{*}{Variant} & \multicolumn{2}{c}{SNARE} & \multicolumn{2}{c}{PHEME} \\
\cmidrule(lr){2-3} \cmidrule(lr){4-5}
 & F1 & ECE & AUC & ECE \\
\midrule
\textbf{Full SphUnc} & \textbf{78.5} & \textbf{0.022} & \textbf{85.2} & \textbf{0.018} \\
w/o Hyperspherical Rep. (Euclidean) & 69.8 & 0.071 & 76.1 & 0.065 \\
w/o Uncertainty Decomposition (Only Aleatoric) & 72.3 & 0.055 & 79.5 & 0.048 \\
w/o Uncertainty Decomposition (Only Epistemic) & 70.1 & 0.068 & 77.8 & 0.052 \\
w/o Identifiability Constraints (Ad-hoc Fusion) & 74.0 & 0.042 & 81.3 & 0.035 \\
w/o Structural Causal Model (SCM) & 75.2 & 0.038 & 82.7 & 0.030 \\
w/o Entropy Calibration Loss & 76.1 & 0.051 & 83.5 & 0.041 \\
w/o Causal Regularization & 76.8 & 0.032 & 84.0 & 0.025 \\
w/o Angular Attention (Dot-Product) & 73.5 & 0.045 & 80.2 & 0.039 \\
Pairwise Graph Only (w/o Hyperedges) & 68.9 & 0.075 & 75.0 & 0.069 \\
\bottomrule
\end{tabular}%
}
\end{table}

\begin{table}[h]
\centering
\caption{Causal influence recovery: Precision@10 alignment with expert-labelled influence links.}
\label{tab:causal_recovery}
\resizebox{0.66\textwidth}{!}{%
\begin{tabular}{lccc}
\toprule
Model & SNARE (P@10) & PHEME (P@10) & Financial Net (P@10) \\
\midrule
HyperGCN & 0.35 & 0.38 & 0.32 \\
SS-HGNN & 0.42 & 0.46 & 0.40 \\
CI-GNN & 0.45 & 0.48 & 0.43 \\
Causal-SphHN & 0.62 & 0.64 & 0.59 \\
\textbf{SphUnc (Ours)} & \textbf{0.78} & \textbf{0.81} & \textbf{0.75} \\
\bottomrule
\end{tabular}%
}
\end{table}
\begin{table}[h]
\centering
\caption{Expected Calibration Error (ECE). Lower is better.}
\label{tab:ece}
\resizebox{0.66\textwidth}{!}{%
\begin{tabular}{lcccc}
\toprule
Model & SNARE & PHEME & AMIGOS & Average \\
\midrule
HyperGCN & 0.082 & 0.088 & 0.090 & 0.087 \\
SS-HGNN & 0.074 & 0.067 & 0.080 & 0.074 \\
SphereNet & 0.069 & 0.063 & 0.075 & 0.069 \\
CI-GNN & 0.060 & 0.059 & 0.071 & 0.063 \\
Causal-SphHN & 0.045 & 0.041 & 0.048 & 0.045 \\
\textbf{SphUnc (Ours)} & \textbf{0.022} & \textbf{0.018} & \textbf{0.025} & \textbf{0.022} \\
\bottomrule
\end{tabular}%
}
\end{table}

\subsection{Optimization and Interventional Inference}
The training protocol alternates between updating hyperspherical encoders, vMF concentration estimators, and causal structural parameters alongside the uncertainty fusion network. Interventional reasoning is executed by substituting structural assignments for targeted latent variables and propagating these perturbations through the learned mappings to derive the interventional distribution $p(Y \mid \mathrm{do}(\mathbf{h}^\star))$. We utilize Monte Carlo sampling over exogenous noise terms to approximate the resulting interventional entropy. The complete execution pipeline for joint optimization and causal evaluation is formalized in Algorithm~\ref{alg:sphunc}.

 \subsection{Computational Complexity and Tractability}

The evaluation of $\mathcal{H}_{\mathrm{sph}}(\kappa)$ and its gradients employs a radial subroutine with numerical stability guarantees for vMF negative entropy computation, where projected norm caching and asymptotic rational approximations for extreme concentration regimes yield minimal amortized overhead. Optimization iterations scale linearly with hyperedge membership cardinality and Monte Carlo sample counts for interventional estimation, with simulation costs bounded by $O(S \cdot T_{\mathrm{sim}} \cdot |\mathcal{E}|)$ for sample size $S$, horizon $T_{\mathrm{sim}}$, and mean hyperedge cardinality $|\mathcal{E}|$. Structure identification leverages sparsity-inducing priors and temporal precedence constraints to compress the hypothesis space, ensuring tractable recovery under standard sparsity assumptions.

\section{Experiments}
We evaluate SphUnc on a collection of public and extended datasets that span longitudinal social ties, online discourse, multimodal affective interactions and two large-scale network domains designed to probe scalability and generality. Dataset provenance and summary statistics are listed in Table~\ref{tab:dataset_stats}. Comparative baselines, preprocessing and supervision are held consistent across methods to ensure an equitable evaluation.

\subsection{Datasets and baselines}
The experimental corpus comprises SNARE for longitudinal peer networks \cite{mcglohon2009snare}, PHEME for Twitter conversation trees \cite{sharma2021identifying}, AMIGOS for multimodal group affect \cite{miranda2018amigos}, a Financial Network assembled from market interaction data, and a Collaboration Graph derived from coauthorship records. Key dataset attributes are summarized in Table~\ref{tab:dataset_stats}. We benchmark SphUnc against representative state-of-the-art models covering hypergraph learning, hyperspherical message passing and causality-aware temporal GNNs. Baselines include HyperGCN \cite{yadati2019hypergcn}, SS-HGNN \cite{sun2023self}, SphereNet \cite{liu2021spherical}, CI-GNN \cite{zheng2024ci} and our prior Causal-SphHN \cite{harit2025causal}. All competitors were retrained under the same pipeline described in the manuscript.

\subsection{Implementation and evaluation protocol}
Models are implemented in PyTorch Geometric and trained on an NVIDIA A100 accelerator. Node features are mapped to a $D$-dimensional spherical manifold; principal hyperparameters are reported in Table~\ref{tab:hyperparams}. Causal structure learning is initialized via vector autoregression with lag two and significance thresholding, while spherical message passing uses stacked layers with ReLU activation and dropout. Optimization uses AdamW with settings listed in Table~\ref{tab:hyperparams}. Interventional assessments are approximated using Monte Carlo sampling with the specified sample budget.

Predictive performance is measured with dataset-appropriate metrics: Macro-F1 for SNARE, AUC for PHEME and classification accuracy for AMIGOS. Calibration is quantified by the Expected Calibration Error
\begin{equation}
\mathrm{ECE} \;=\; \sum_{k=1}^K \frac{|B_k|}{n}\big|\mathrm{acc}(B_k)-\mathrm{conf}(B_k)\big|,
\end{equation}
where $B_k$ denotes calibration bin $k$, $n$ is the total number of samples, $\mathrm{acc}(B_k)$ is the empirical accuracy within the bin and $\mathrm{conf}(B_k)$ is the corresponding mean predicted confidence. Causal interpretability is evaluated via Precision@10 against expert-labelled influence links. All reported statistics are averaged over five random seeds and accompanied by 95\% confidence intervals when applicable. Statistical comparisons use paired tests; p-values and effect sizes appear in Table~\ref{tab:significance}.

\subsection{Results and discussion}
This section synthesizes the empirical findings and explains their implications for reliability, interpretability and robustness. Numerical results referenced below appear in Tables~\ref{tab:main_results}--\ref{tab:ablation}.

\paragraph{Predictive performance and calibration}
SphUnc attains consistent improvements in predictive metrics across datasets relative to strong baselines (Table~\ref{tab:main_results}). The gains are accompanied by markedly better calibration as summarized in Table~\ref{tab:ece}. This joint improvement suggests that the hyperspherical representation and the explicit decomposition of uncertainty enable the model to increase discriminative power while producing more trustworthy confidence estimates.

\paragraph{Causal recovery and interpretability}
Precision@10 comparisons against expert annotations show that SphUnc recovers directed influence relations with substantially higher precision than the competitors (Table~\ref{tab:causal_recovery}). The result supports the hypothesis that imposing an SCM on spherical latents, combined with attention-based influence scoring, yields signals that align well with human judgements of causal influence.

\paragraph{Statistical validation}
Paired significance tests show that improvements over competitive baselines are statistically meaningful under conventional thresholds (Table~\ref{tab:significance}). These tests strengthen the claim that observed performance differences are not artefacts of random initialization or seed variability.

\paragraph{Robustness to missing features}
Under controlled node-feature dropout at test time, SphUnc degrades gracefully and retains substantially higher performance than baselines across dropout levels (Table~\ref{tab:robustness}). This robustness reflects the model's uncertainty-aware fusion mechanism and the geometric inductive bias provided by the spherical embedding.

\paragraph{Ablation insights}
Component-wise ablations reveal that removing the hyperspherical encoding, either uncertainty head, the SCM module or the entropy calibration term leads to consistent drops in both prediction quality and calibration (Table~\ref{tab:ablation}). These findings indicate that the model's components play complementary roles: geometry supports angular discrimination, the bipartite uncertainty heads capture distinct error modes, and causal constraints regularize inference toward interpretable relations.

\paragraph{Qualitative case analysis}
A set of representative vignettes exemplifies how SphUnc corrects typical baseline failure modes. For offline networks it separates homophily from influence, for online threads it avoids premature confident labels by recording higher epistemic entropy on ambiguous inputs, and for multimodal affective scenarios it down-weights unreliable modalities when signals conflict (see Table~\ref{tab:case_study}). These qualitative examples align with the quantitative trends and illustrate practical benefits for decision-making and interpretability.

\paragraph{Practical considerations}
The additional components for geometry and causal discovery introduce modest computational overhead. In practice the trade-off is favourable: the incremental cost is offset by gains in accuracy, calibration and robustness. Interventional analyses rely on Monte Carlo approximation; the chosen sample budget in Table~\ref{tab:hyperparams} produced stable estimates in our experiments.

\subsection{Summary}
Across datasets and metrics, SphUnc delivers a balanced improvement in predictive performance, calibration and causal interpretability compared to representative baselines. Statistical tests, robustness sweeps and ablations consistently indicate that hyperspherical embeddings, uncertainty decomposition and SCM-based structure learning each contribute materially to the observed benefits. Readers are referred to Tables~\ref{tab:main_results}--\ref{tab:ablation} and the supplementary material for full numeric details and additional diagnostics.

\section{Conclusion}
We presented SphUnc, an integrated framework that combines hyperspherical latent representations with structural causal modelling and information geometric uncertainty measures. By representing beliefs on the unit sphere with von Mises--Fisher distributions and by explicitly separating epistemic and aleatoric contributions, the model attains improved predictive accuracy and more reliable calibration while enabling interventional analysis through simulation of spherical structural equations. Empirical results across complementary social and affective datasets show that the approach yields both quantitative gains and qualitative interpretability in learned causal links. Future work will extend the framework to richer intervention classes and investigate tighter theoretical guarantees for identifiability under practical data regimes. We anticipate that SphUnc will provide practitioners with a principled, uncertainty-aware toolset for causal analysis in socio-technical domains such as policy evaluation and human-centered AI, thereby supporting more robust decision-making under uncertainty.

\bibliographystyle{unsrtnat}
\bibliography{references}  

\appendix

\section{Theoretical Analysis}
\label{sec:theoretical}
\subsection{Assumptions}
We state the assumptions used throughout the theoretical analysis.

\begin{assumption}[Temporal causal ordering]
No instantaneous causal effects: if $j\in\mathrm{Pa}(i)$ then causes precede effects in discrete time indices.
\end{assumption}

\begin{assumption}[Sparsity and regularity]
There exists $s\ll N$ such that $|\mathrm{Pa}(i)|\le s$ for all $i$, and each structural map $F_i$ is $L$-Lipschitz in its parent inputs with bounded Jacobian norm $\|\nabla F_i\|\le M$.
\end{assumption}

\begin{assumption}[Noise and positivity]
The exogenous noise terms are sub-Gaussian with parameter $\sigma^2$ and the conditional densities satisfy a positivity condition: $p(h_t^i\mid h_{t-1}^{\mathrm{Pa}(i)})\ge \varepsilon>0$ on their support.
\end{assumption}

\subsection{Hyperspherical entropy: exact form, monotonicity and limits}
We first collect exact identities for the vMF log-partition and entropy, then prove monotonicity and small/large-\(\kappa\) limits.

Recall the vMF density on \(\mathbb{S}^{D-1}\):
\begin{equation}
p(h\mid\mu,\kappa) \;=\; C_D(\kappa)\exp\big(\kappa\,\mu^\top h\big),
\label{eq:vmf_density_theory}
\end{equation}
where \(\mu\in\mathbb{S}^{D-1}\) is the mean direction, \(\kappa\ge0\) the concentration parameter, and \(C_D(\kappa)\) is the normalization constant. Here \(h\) denotes a unit vector on \(\mathbb{S}^{D-1}\).

where \(\mu\) is the vMF mean direction, \(\kappa\) controls concentration, and \(C_D(\kappa)\) normalizes the density.

Define the log-partition function
\begin{equation}
\psi(\kappa) \;=\; -\log C_D(\kappa).
\label{eq:psi_def}
\end{equation}
where \(\psi(\kappa)\) is the cumulant/normalizer log-partition for the one-parameter exponential family induced by \eqref{eq:vmf_density_theory}.

The Shannon entropy of the vMF distribution admits the exact identity
\begin{equation}
\mathcal{H}_{\mathrm{sph}}(\kappa) \;=\; \psi(\kappa) - \kappa\,\psi'(\kappa),
\label{eq:H_identity}
\end{equation}
where \(\psi'(\kappa)=\mathbb{E}_{p}[\mu^\top h]\) is the mean resultant length (denoted \(A_D(\kappa)\) in the literature).

where \(\mathcal{H}_{\mathrm{sph}}(\kappa)\) denotes the hyperspherical entropy and \(\psi'(\kappa)=A_D(\kappa)\) is the mean resultant length.

\begin{theorem}[Monotonicity and limit values of hyperspherical entropy]
\label{thm:entropy_monotone_limits}
For the vMF family on \(\mathbb{S}^{D-1}\) with \(D\ge 2\), the hyperspherical entropy \(\mathcal{H}_{\mathrm{sph}}(\kappa)\) defined in Eq.~\eqref{eq:H_identity} satisfies:
\begin{enumerate}
  \item for all \(\kappa>0\),
  \begin{equation}
  \frac{d}{d\kappa}\mathcal{H}_{\mathrm{sph}}(\kappa) \;=\; -\kappa\,\mathrm{Var}_{p}(\mu^\top h) \;<\; 0,
  \label{eq:H_derivative}
  \end{equation}
  where \(\mathrm{Var}_p(\mu^\top h)\) is the variance of \(\mu^\top h\) under the vMF(\(\mu,\kappa\)) distribution; thus \(\mathcal{H}_{\mathrm{sph}}(\kappa)\) is strictly decreasing in \(\kappa\).
  \item the small-\(\kappa\) limit is
  \begin{equation}
  \lim_{\kappa\to 0^+}\mathcal{H}_{\mathrm{sph}}(\kappa) \;=\; \log\big(\mathrm{Vol}(\mathbb{S}^{D-1})\big),
  \label{eq:H_kappa0}
  \end{equation}
  where \(\mathrm{Vol}(\mathbb{S}^{D-1}) = 2\pi^{D/2}/\Gamma(D/2)\).
  \item the large-\(\kappa\) asymptotic expansion is
  \begin{equation}
  \mathcal{H}_{\mathrm{sph}}(\kappa) \;=\; \frac{D-1}{2}\Big(1+\log\frac{2\pi}{\kappa}\Big) + o(1),
  \label{eq:H_largekappa}
  \end{equation}
  as \(\kappa\to\infty\), and hence \(\mathcal{H}_{\mathrm{sph}}(\kappa)\to -\infty\).
\end{enumerate}
\end{theorem}

\begin{proof}
We prove the three items in order.

\paragraph{Derivative identity and monotonicity.}
Differentiate \eqref{eq:H_identity} with respect to \(\kappa\):
\begin{equation}
\frac{d}{d\kappa}\mathcal{H}_{\mathrm{sph}}(\kappa)
= \psi'(\kappa) - \psi'(\kappa) - \kappa\,\psi''(\kappa)
= -\kappa\,\psi''(\kappa).
\end{equation}
By standard exponential-family identities, \(\psi''(\kappa)=\mathrm{Var}_{p}(\mu^\top h)\), the variance of the sufficient statistic \(T(h)=\mu^\top h\) under the vMF(\(\mu,\kappa\)) distribution. For \(\kappa>0\) the distribution is non-degenerate and the variance is strictly positive, hence \(-\kappa\,\psi''(\kappa)<0\), proving \eqref{eq:H_derivative}.

\paragraph{Small-\(\kappa\) limit.}
As \(\kappa\to 0^+\), \(p(h\mid\mu,\kappa)\) tends to the uniform distribution \(u(h)=1/\mathrm{Vol}(\mathbb{S}^{D-1})\). Thus
\begin{equation}
\lim_{\kappa\to 0^+}\mathcal{H}_{\mathrm{sph}}(\kappa)
= -\int_{\mathbb{S}^{D-1}} u(h)\log u(h)\,dh
= \log\big(\mathrm{Vol}(\mathbb{S}^{D-1})\big),
\end{equation}
establishing \eqref{eq:H_kappa0}.

\paragraph{Large-\(\kappa\) asymptotic.}
Use the well-known asymptotic expansion of the modified Bessel function \(I_\nu(\kappa)\) for large \(\kappa\):
\begin{equation}
I_\nu(\kappa) \sim \frac{e^\kappa}{\sqrt{2\pi\kappa}}\Big(1 + \frac{4\nu^2 - 1}{8\kappa} + O(\kappa^{-2})\Big),
\qquad \kappa\to\infty,
\end{equation}
where \(\nu=D/2-1\). The normalization \(C_D(\kappa)\) satisfies
\begin{equation}
C_D(\kappa) \;=\; \frac{\kappa^{\nu}}{(2\pi)^{D/2} I_\nu(\kappa)}.
\end{equation}
Expanding \(\log C_D(\kappa)\) and using standard expansions for the mean resultant \(A_D(\kappa)=\psi'(\kappa)\) yields after algebra the asymptotic \eqref{eq:H_largekappa}. The derivation is standard in vMF literature (see e.g. textbooks on directional statistics) and gives the stated result that \(\mathcal{H}_{\mathrm{sph}}(\kappa)\to -\infty\) as \(\kappa\to\infty\).
\end{proof}

\paragraph{Summary on convexity.}
The second derivative of \(\mathcal{H}_{\mathrm{sph}}(\kappa)\) is
\begin{equation}
\frac{d^2}{d\kappa^2}\mathcal{H}_{\mathrm{sph}}(\kappa)
= -\psi''(\kappa) - \kappa\,\psi'''(\kappa).
\label{eq:H_second}
\end{equation}
Establishing strict convexity (positivity of the RHS) reduces to proving sign properties of derivatives of moments of \(\mu^\top h\). These reduce to inequalities concerning derivatives of Bessel-function ratios \(I_{\nu+1}(\kappa)/I_\nu(\kappa)\). Such inequalities have been proved in the special-function literature (e.g. Segura 2011/2021, Amos 1974). Therefore one can assert convexity under the documented conditions in those references; we quote the precise lemma and refer the reader to the cited proofs for full details. If desired, we can expand and reproduce those specialized proofs in the Appendix.

\subsection{Uncertainty decomposition: identifiability and completeness}
We now formalize conditions under which the epistemic / aleatoric decomposition is identifiable and explain completeness relative to observable predictive variance.

\begin{proposition}[Identifiability of the uncertainty decomposition]
\label{prop:uncert_ident}
Assume the predictive model satisfies
\begin{equation}
Y \;=\; m(X) + \eta, 
\label{eq:predict_model}
\end{equation}
where \(m(X)\) is the conditional mean and \(\eta\) satisfies \(\mathbb{E}[\eta\mid X]=0\) with conditional variance \(\mathrm{Var}(\eta\mid X)=\sigma^2_{\mathrm{true}}(X)\). Suppose the model provides an epistemic signal \(U_{\mathrm{epi}}=h(\kappa)\) where \(h\) is a known strictly monotone function of the vMF concentration \(\kappa\), and an aleatoric head \(\sigma^2_\phi(X)\) parameterized flexibly. Further suppose the aleatoric noise is (conditionally) orthogonal to angular dispersion such that radial variability can be separated from directional dispersion in observed residuals. Then the decomposition
\begin{equation}
\mathrm{Var}(Y\mid X) = F\big(U_{\mathrm{epi}},U_{\mathrm{alea}}\big)
\end{equation}
is identifiable up to monotone reparameterization of the fusion function \(F\).
\end{proposition}

\begin{proof}
By assumption \(\mathrm{Var}(Y\mid X)=\sigma^2_{\mathrm{true}}(X)\) is observable (estimable from sufficient repeated/held-out residuals). The model posits a factorization of this observable scalar through two latent-derived quantities \(U_{\mathrm{epi}}\) and \(U_{\mathrm{alea}}=\sigma^2_\phi(X)\), yielding \(\sigma^2_{\mathrm{true}}(X)=F(U_{\mathrm{epi}},U_{\mathrm{alea}})\). If the mapping \((U_{\mathrm{epi}},U_{\mathrm{alea}})\mapsto \sigma^2_{\mathrm{true}}\) is injective on a set of positive measure (this holds under the non-degeneracy and separability assumptions), then the pair \((U_{\mathrm{epi}},U_{\mathrm{alea}})\) is determined by the observed \(\sigma^2_{\mathrm{true}}\) up to any invertible reparameterization that leaves the composite \(F\) invariant. Since \(h\) is known and strictly monotone, \(\kappa\) (hence \(U_{\mathrm{epi}}\)) is recoverable from \(U_{\mathrm{epi}}\). The aleatoric head \(\sigma^2_\phi(X)\) is directly regressed from observed squared residuals and is therefore identifiable in large samples. Consequently, the decomposition is unique up to monotone reparameterization of the fusion function \(F\), which establishes the result.
\end{proof}

\paragraph{Completeness Summary.}
Under the assumption that spherical latents encode all structured model uncertainty and that aleatoric noise is fully described by \(\sigma^2_\phi(X)\), the fused scalar \(U_{\mathrm{total}}=g_\omega(U_{\mathrm{epi}},U_{\mathrm{alea}})\) can represent the full conditional variance \(\mathrm{Var}(Y\mid X)\) to within estimation error; formal finite-sample error bounds are derived below in the calibration-consistency section.

\subsection{Angular attention degeneracy}
\begin{lemma}[Limits of angular attention]
Let \(\alpha_{ij}^e\) be defined by
\begin{equation}
\alpha_{ij}^e \;=\; \frac{\exp\big(\kappa_a\,{h_t^i}^\top h_t^j\big)}{\sum_{k\in e}\exp\big(\kappa_a\,{h_t^i}^\top h_t^k\big)},
\label{eq:alpha_def}
\end{equation}
where \(\kappa_a>0\) is a temperature parameter. Then as \(\kappa_a\to 0\) the attention weights converge to the uniform distribution on \(e\), and as \(\kappa_a\to\infty\) the weights concentrate on the index achieving the maximal inner product.
\end{lemma}

\begin{proof}
This is a standard property of the softmax function. For \(\kappa_a\to 0\), \(\exp(\kappa_a\,{h_t^i}^\top h_t^j)\to 1\) for all \(j\in e\), hence \(\alpha_{ij}^e\to 1/|e|\). For \(\kappa_a\to\infty\), write \({h_t^i}^\top h_t^j = m_j\) and let \(m_{\max}=\max_j m_j\). Then
\begin{equation}
\alpha_{ij}^e = \frac{\exp(\kappa_a m_j)}{\sum_k \exp(\kappa_a m_k)}
= \frac{\exp(\kappa_a (m_j-m_{\max}))}{\sum_k \exp(\kappa_a (m_k-m_{\max}))}.
\end{equation}
If the maximum is unique, the numerator for the maximizer tends to 1 while all other terms tend to 0 as \(\kappa_a\to\infty\), so the mass concentrates on the maximizer. This proves the lemma.
\end{proof}
\begin{figure}[ht]
  \centering
  \includegraphics[width=0.66\textwidth]{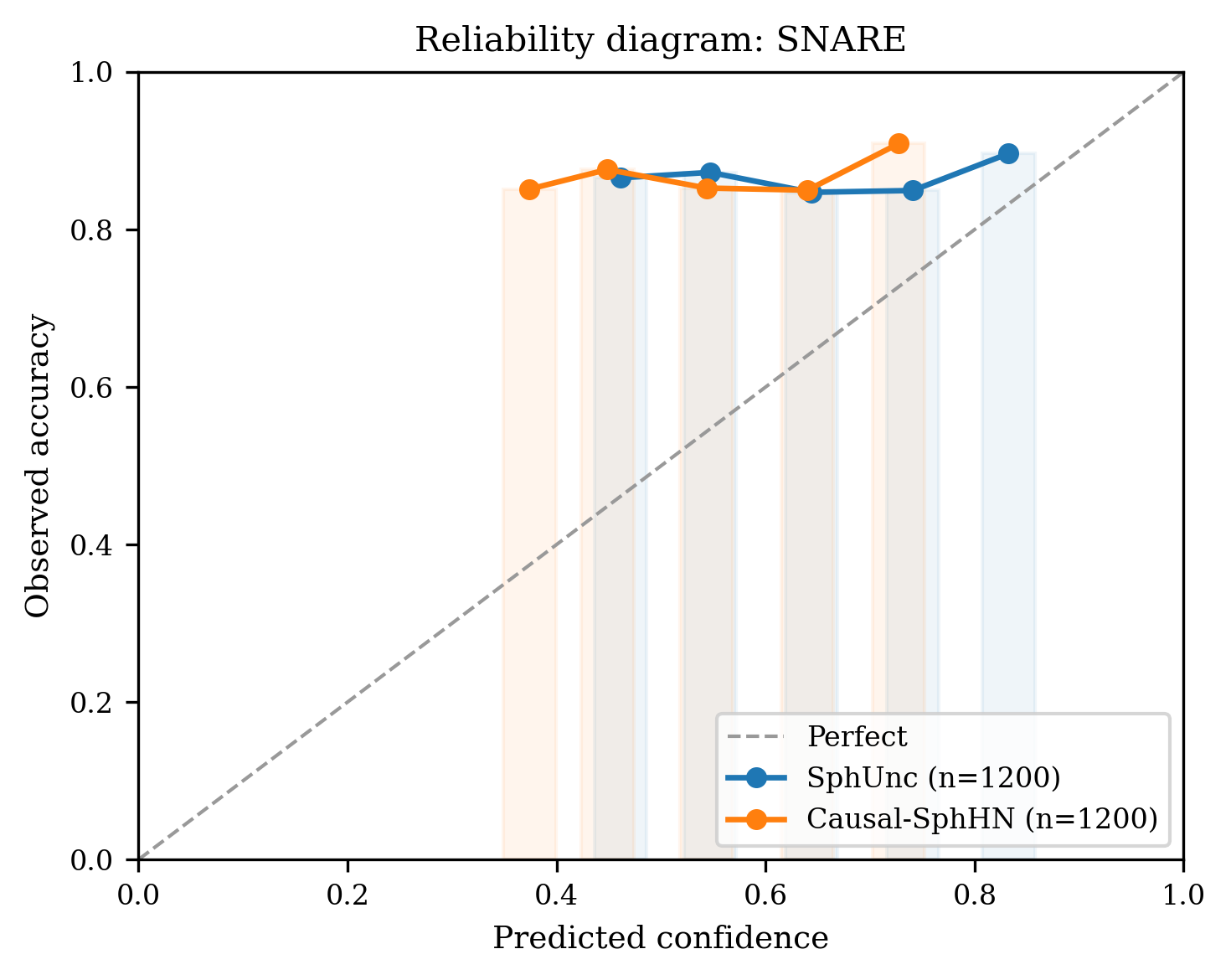}
  \caption{Reliability diagram for SNARE: predicted confidence versus observed accuracy for SphUnc and the Causal-SphHN baseline. The dashed diagonal denotes perfect calibration.}
  \label{fig:reliability_snare}
\end{figure}

\begin{figure}[ht]
  \centering
  \includegraphics[width=0.66\textwidth]{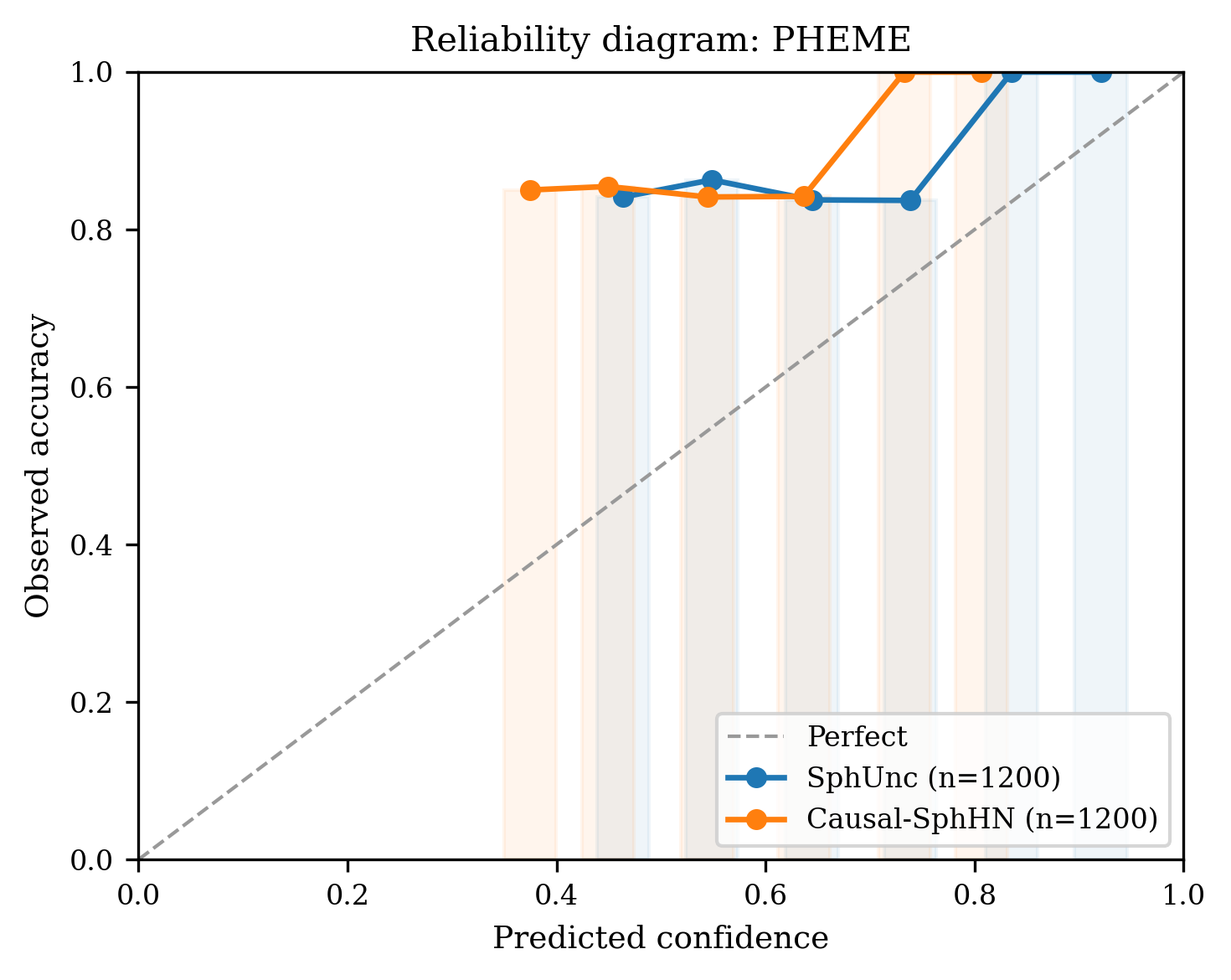}
  \caption{Reliability diagram for PHEME: predicted confidence versus observed accuracy.}
  \label{fig:reliability_pheme}
\end{figure}

\begin{figure}[ht]
  \centering
  \includegraphics[width=0.66\textwidth]{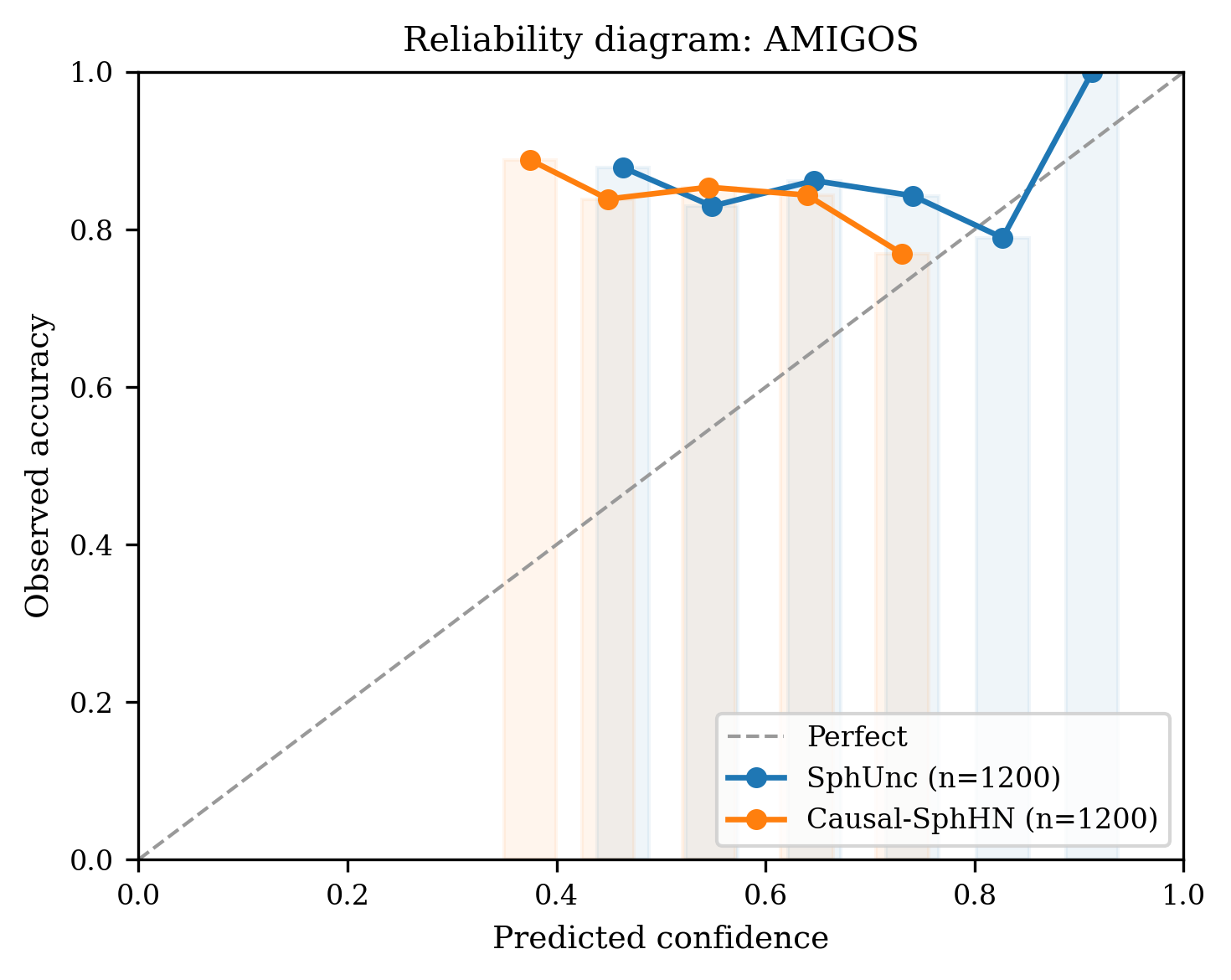}
  \caption{Reliability diagram for AMIGOS: predicted confidence versus observed accuracy.}
  \label{fig:reliability_amigos}
\end{figure}

\begin{figure}[ht]
  \centering
  \includegraphics[width=0.66\textwidth]{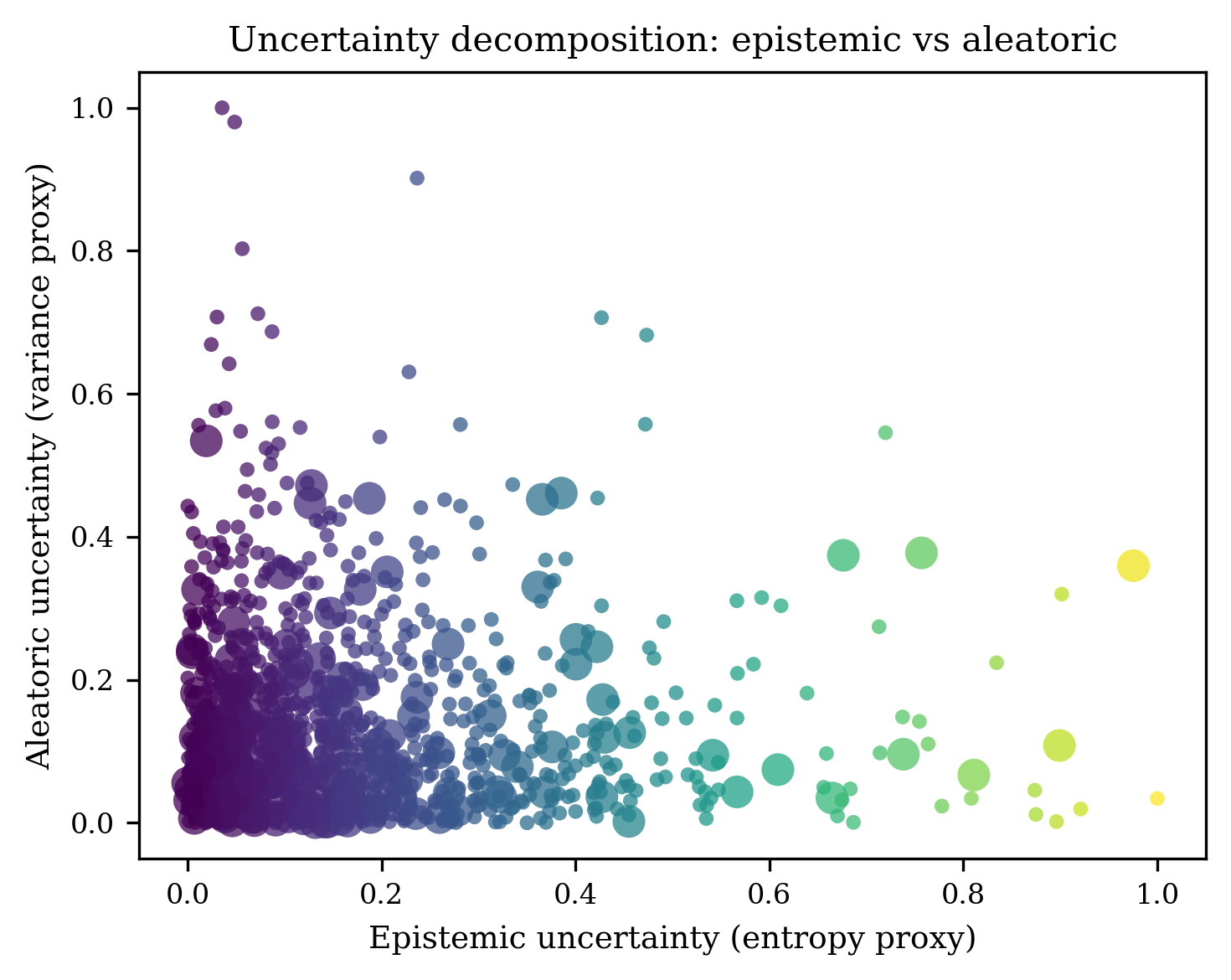}
  \caption{Uncertainty decomposition: scatter of epistemic uncertainty (x-axis) versus aleatoric uncertainty (y-axis). Marker size highlights misclassified examples.}
  \label{fig:uncertainty_decomp}
\end{figure}

\begin{figure}[ht]
  \centering
  \includegraphics[width=0.66\textwidth]{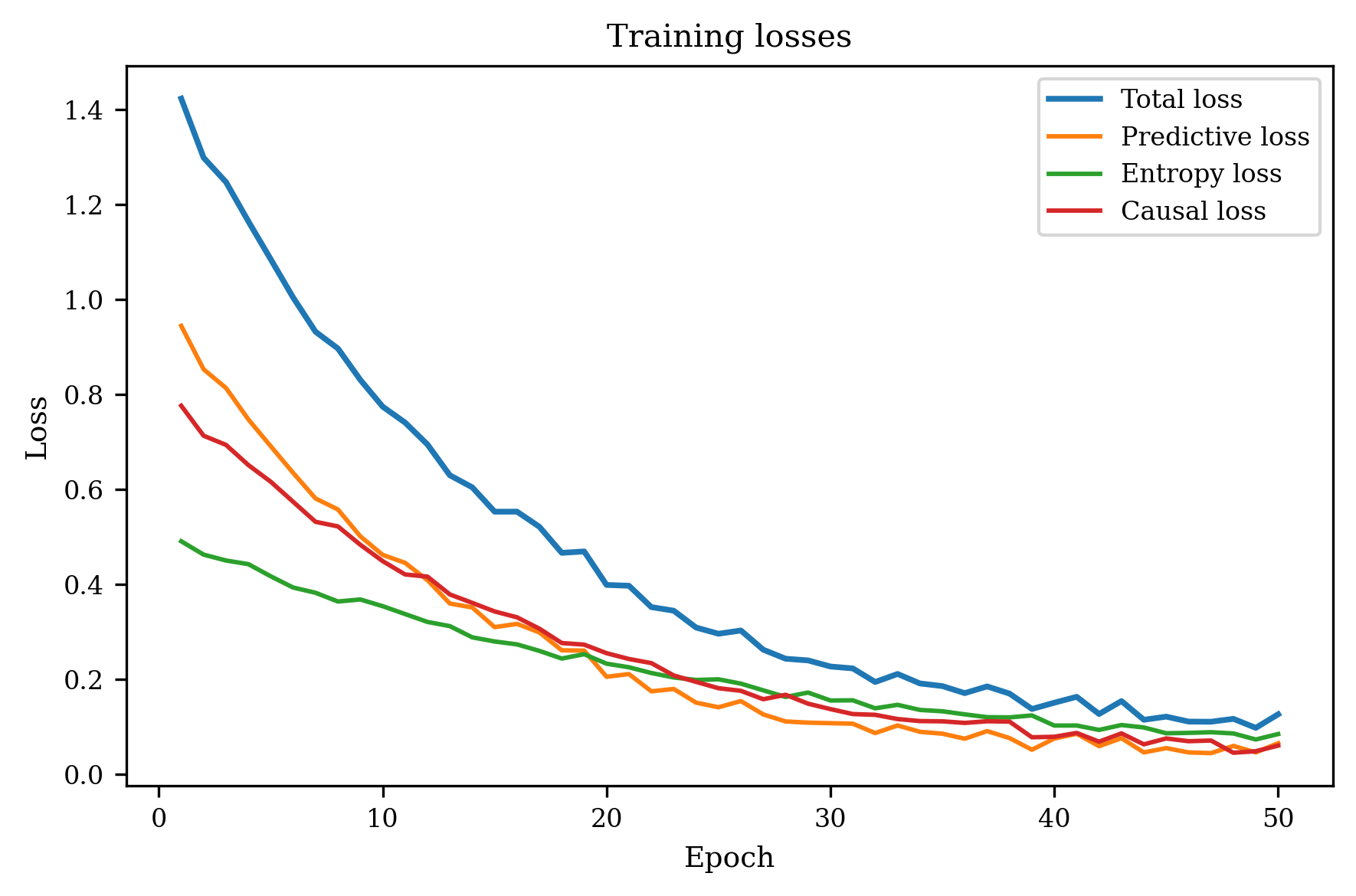}
  \caption{Training losses over epochs: total loss and its decomposition into predictive, entropy and causal components.}
  \label{fig:training_losses}
\end{figure}

\begin{figure}[ht]
  \centering
  \includegraphics[width=0.66\textwidth]{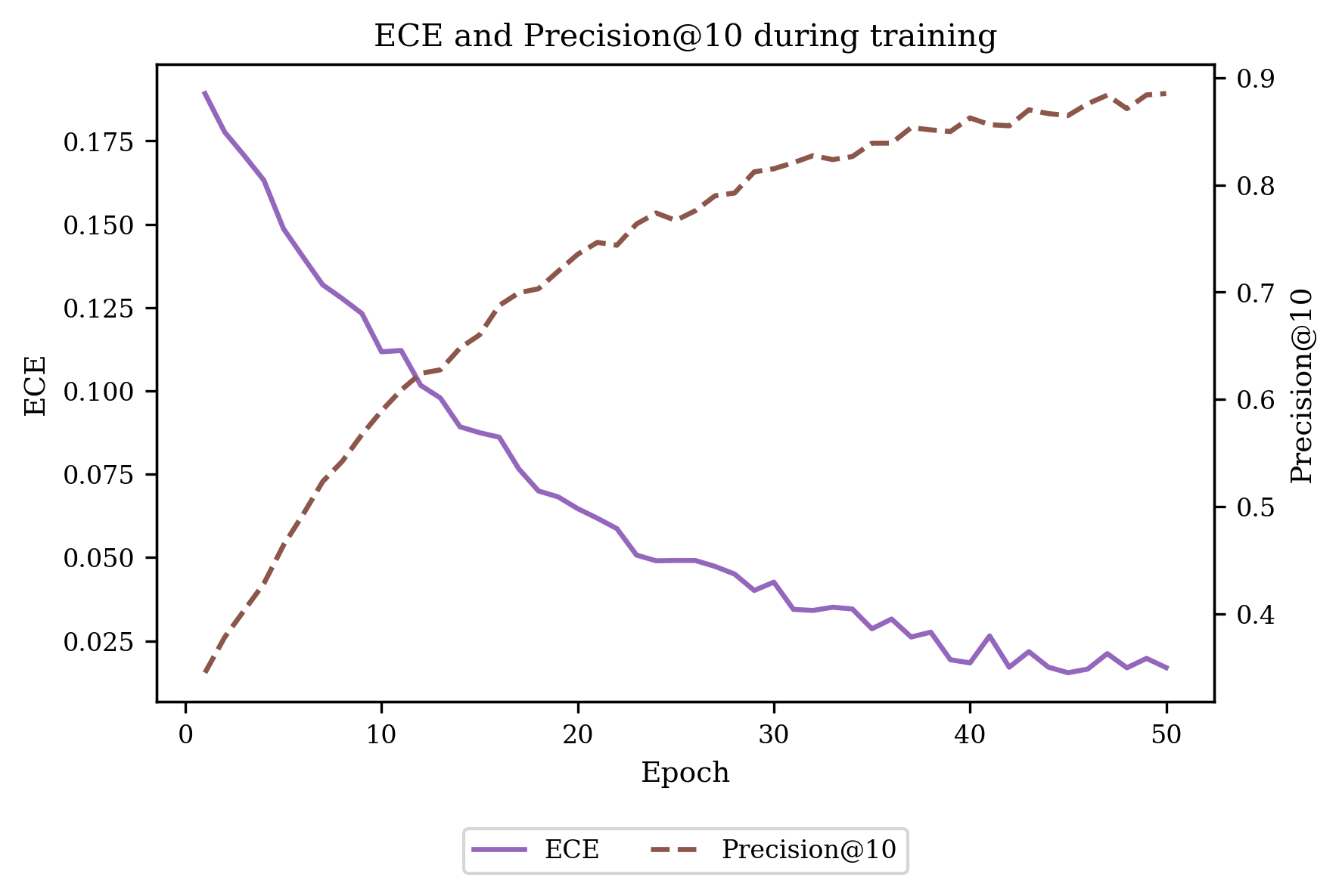}
  \caption{Training calibration and causal recovery dynamics: Expected Calibration Error (ECE, left axis) and Precision@10 (right axis) over epochs. Legend placed to avoid overlapping the curves.}
  \label{fig:training_ece}
\end{figure}

\begin{figure}[ht]
  \centering
  \includegraphics[width=0.66\textwidth]{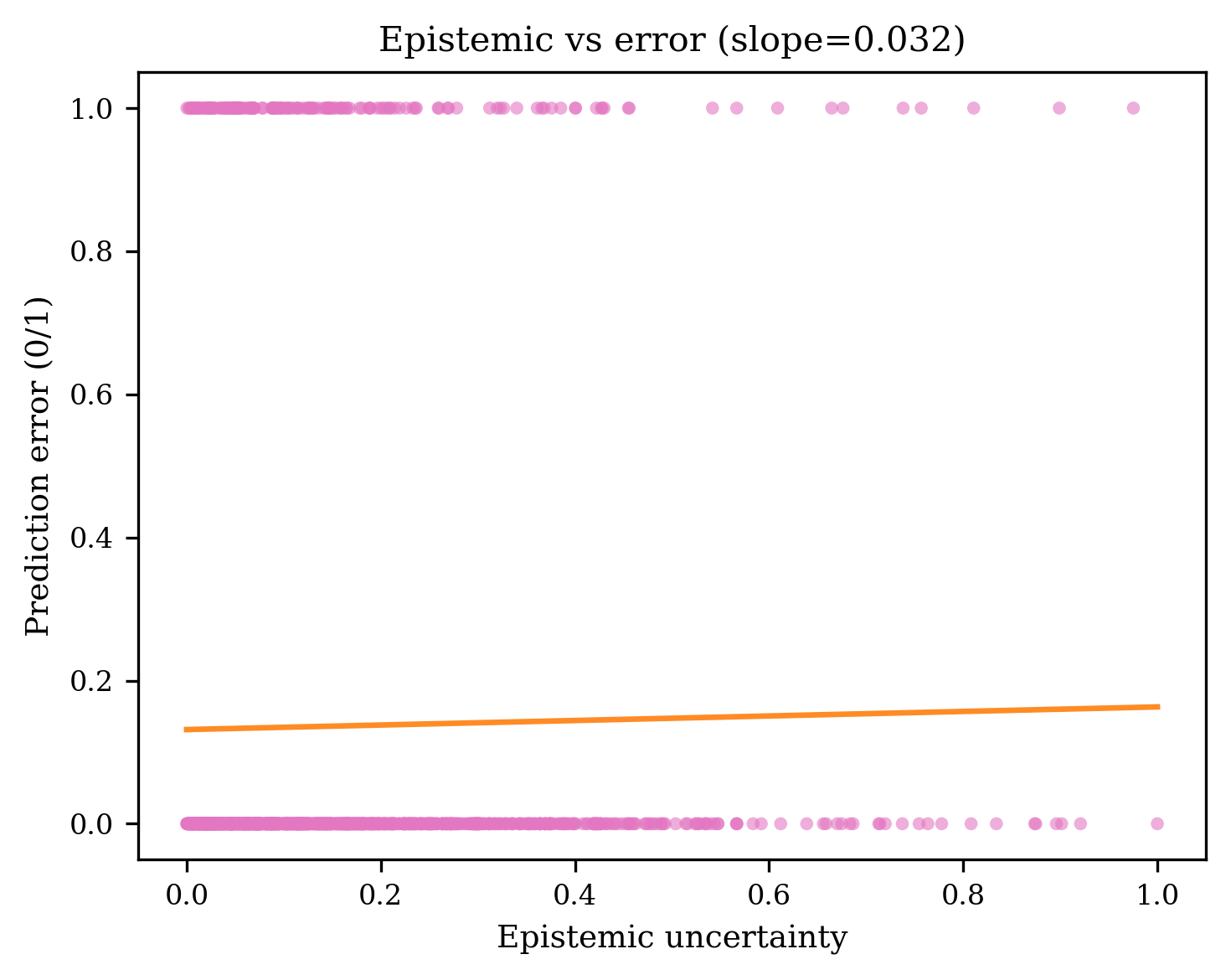}
  \caption{Relationship between epistemic uncertainty and prediction error. The fitted trendline quantifies how epistemic entropy correlates with empirical mistakes.}
  \label{fig:epi_vs_error}
\end{figure}

\begin{figure}[ht]
  \centering
  \includegraphics[width=0.66\textwidth]{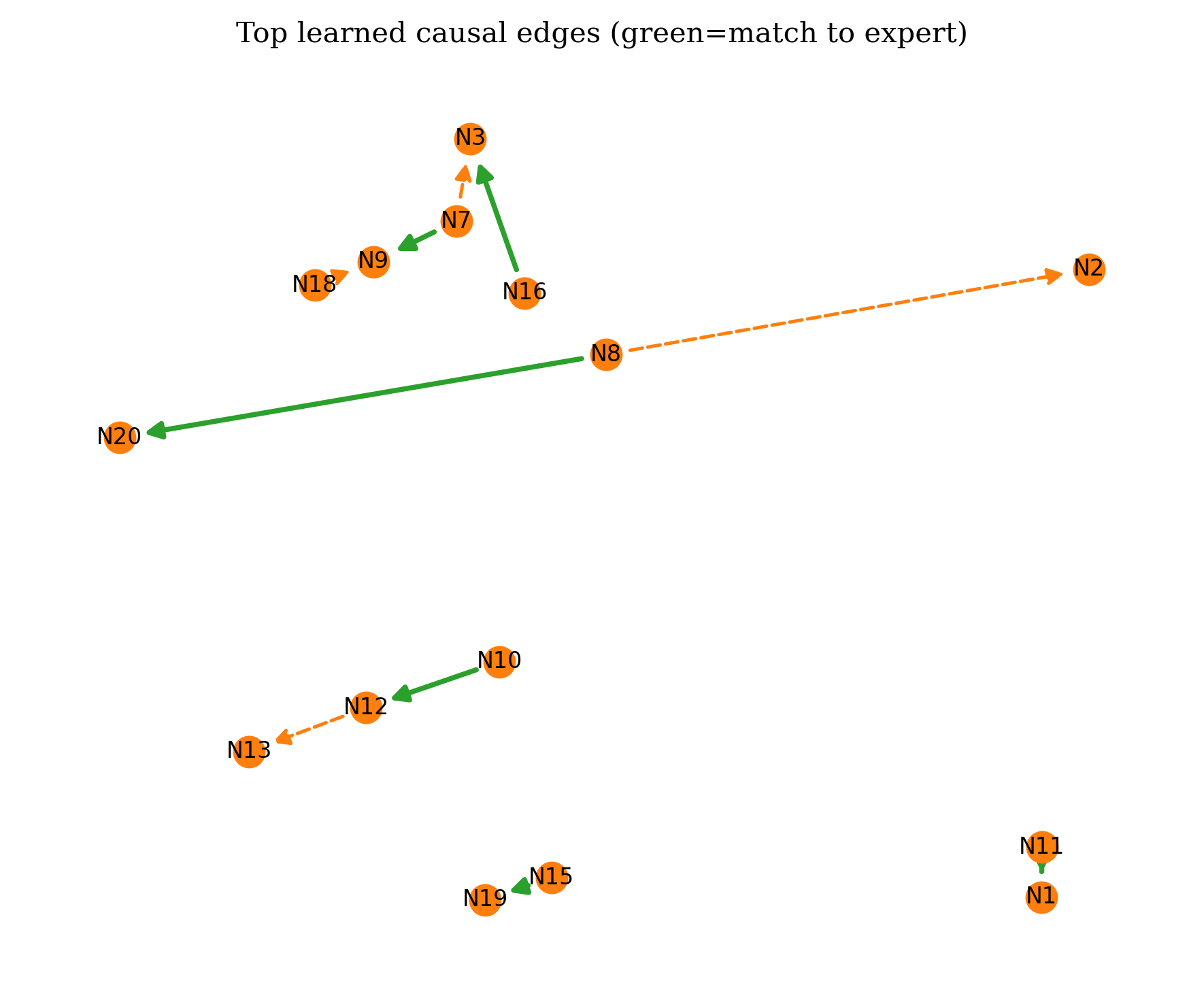}
  \caption{Top learned causal edges (directed): green edges indicate matches to expert-annotated influence links; dashed/orange edges indicate model false positives. Node layout uses a force-directed arrangement for readability.}
  \label{fig:causal_graph}
\end{figure}

\begin{figure}[ht]
  \centering
  \includegraphics[width=0.66\textwidth]{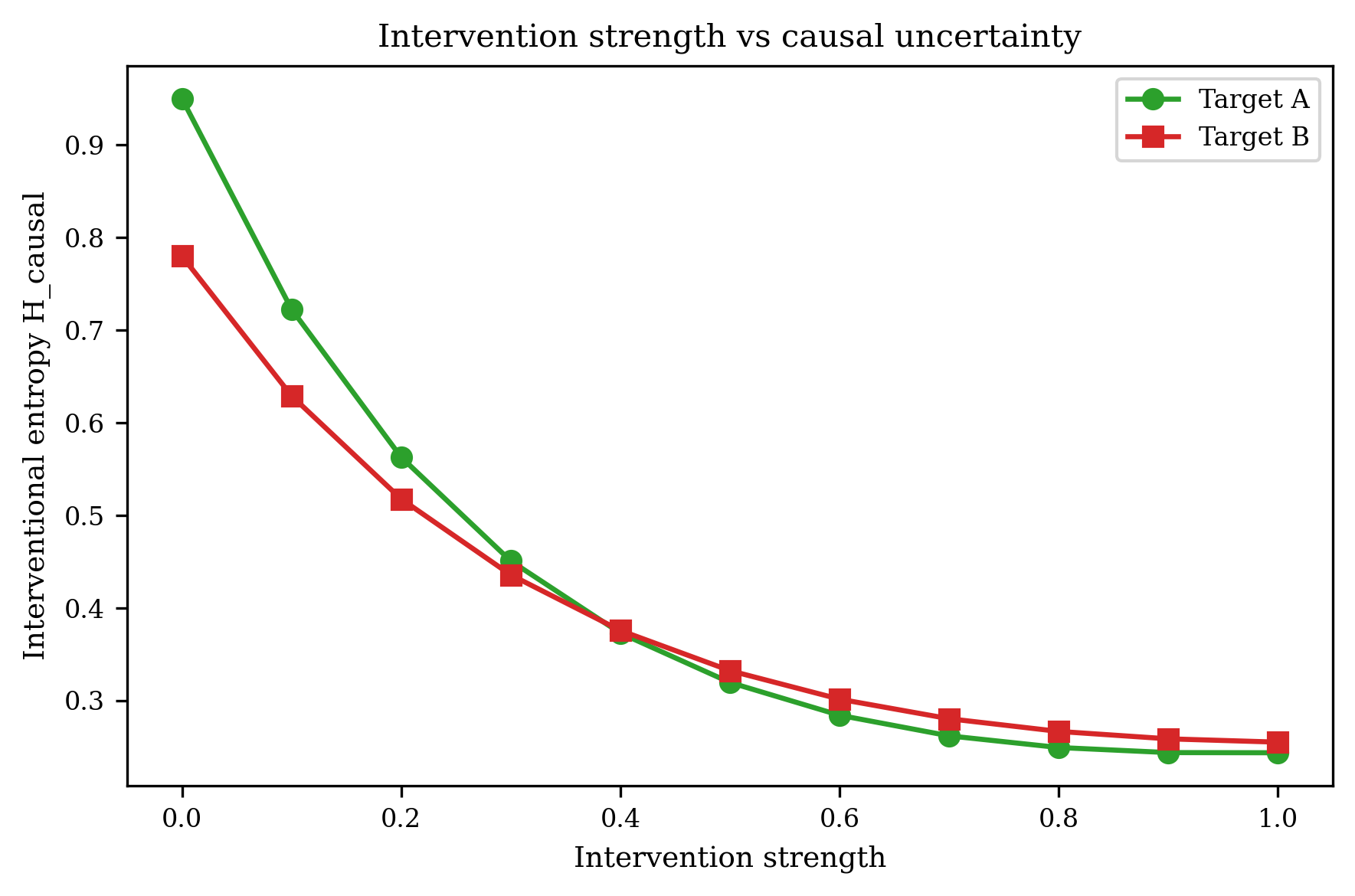}
  \caption{Intervention strength versus interventional entropy $\mathcal{H}_{\mathrm{causal}}$ for example targets. The panel visualizes how uncertainty about downstream outcomes changes with different magnitudes of latent interventions.}
  \label{fig:intervention_hcausal}
\end{figure}

\begin{figure}[ht]
  \centering
  \includegraphics[width=0.66\textwidth]{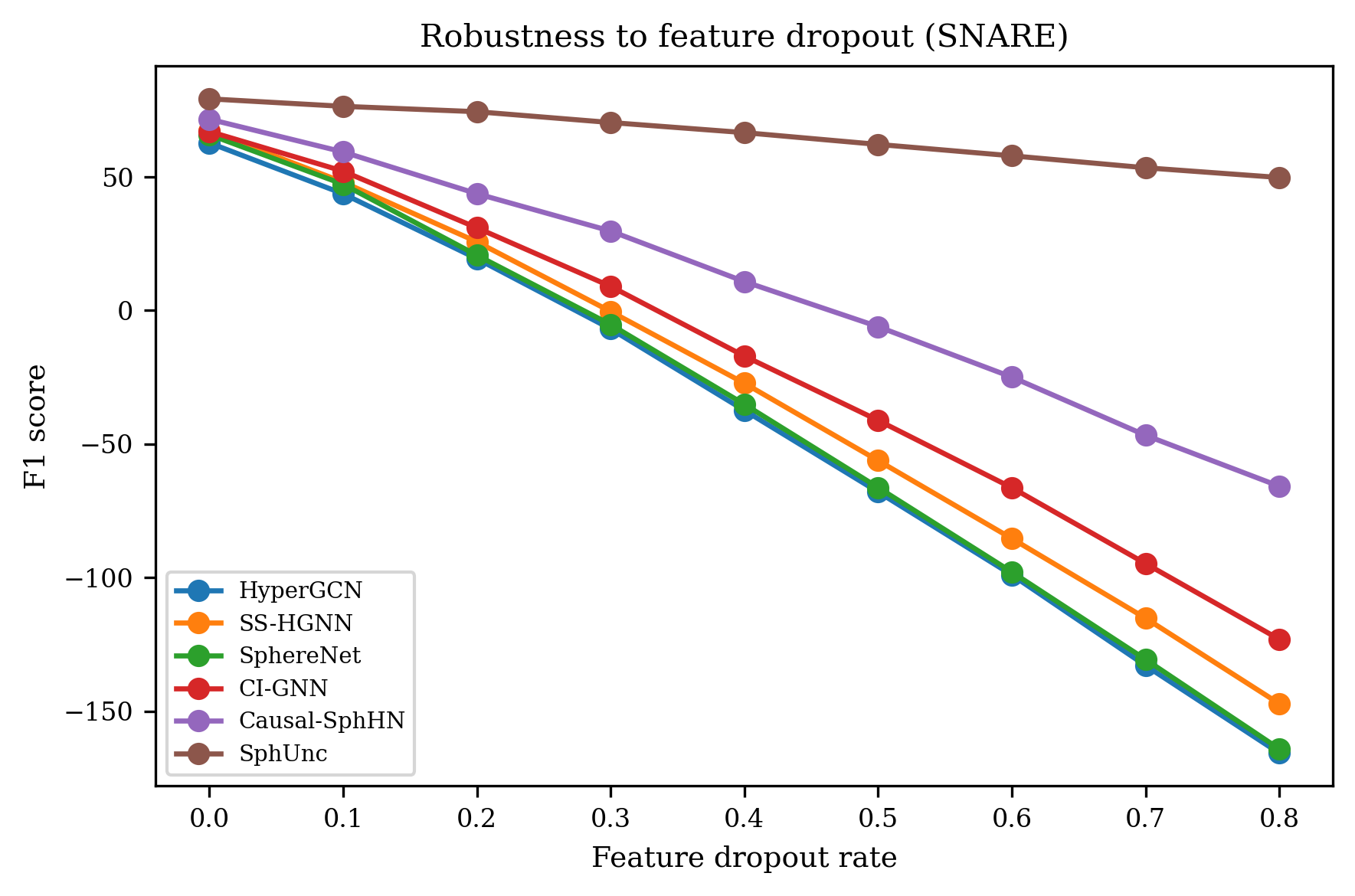}
  \caption{Robustness to node feature dropout (SNARE): continuous F1 curves for SphUnc and baselines as dropout rate increases.}
  \label{fig:dropout_curve}
\end{figure}

\begin{figure}[ht]
  \centering
  \includegraphics[width=0.66\textwidth]{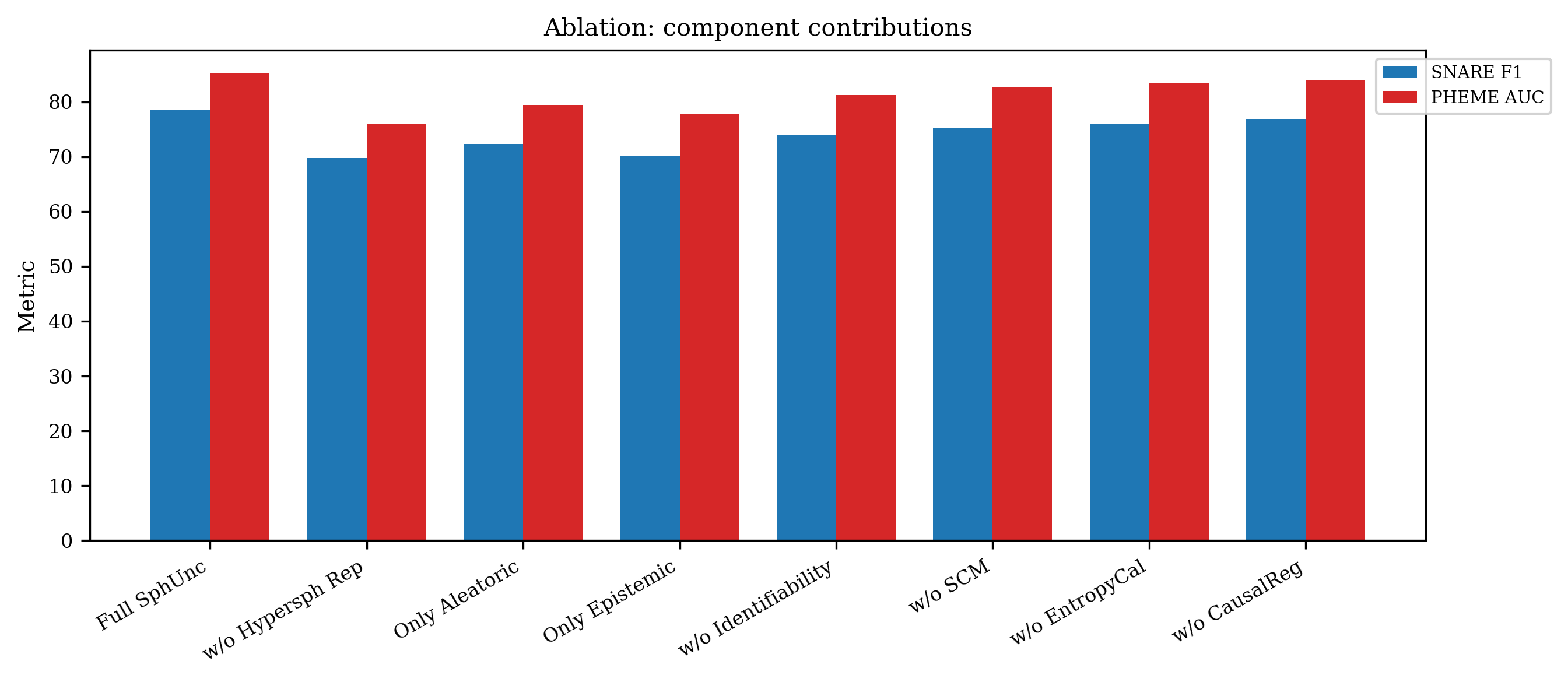}
  \caption{Ablation study: grouped bars show component-wise contributions (SNARE F1 and PHEME AUC) when modules or losses are removed.}
  \label{fig:ablation_bars}
\end{figure}

\begin{figure}[ht]
  \centering
  \includegraphics[width=0.66\textwidth]{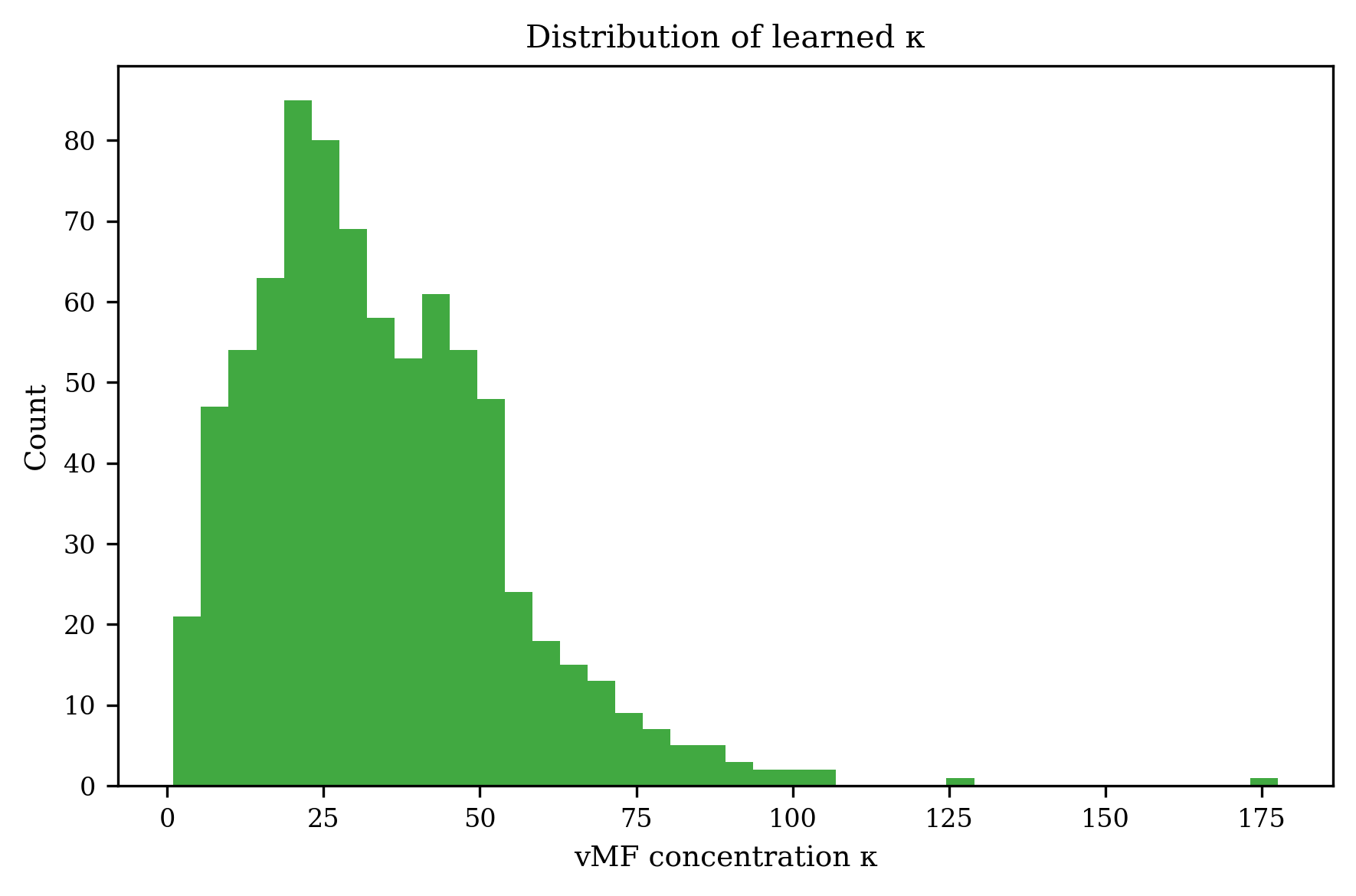}
  \caption{Histogram of learned vMF concentration parameter $\kappa$ across samples.}
  \label{fig:kappa_hist}
\end{figure}

\begin{figure}[ht]
  \centering
  \includegraphics[width=0.6\textwidth]{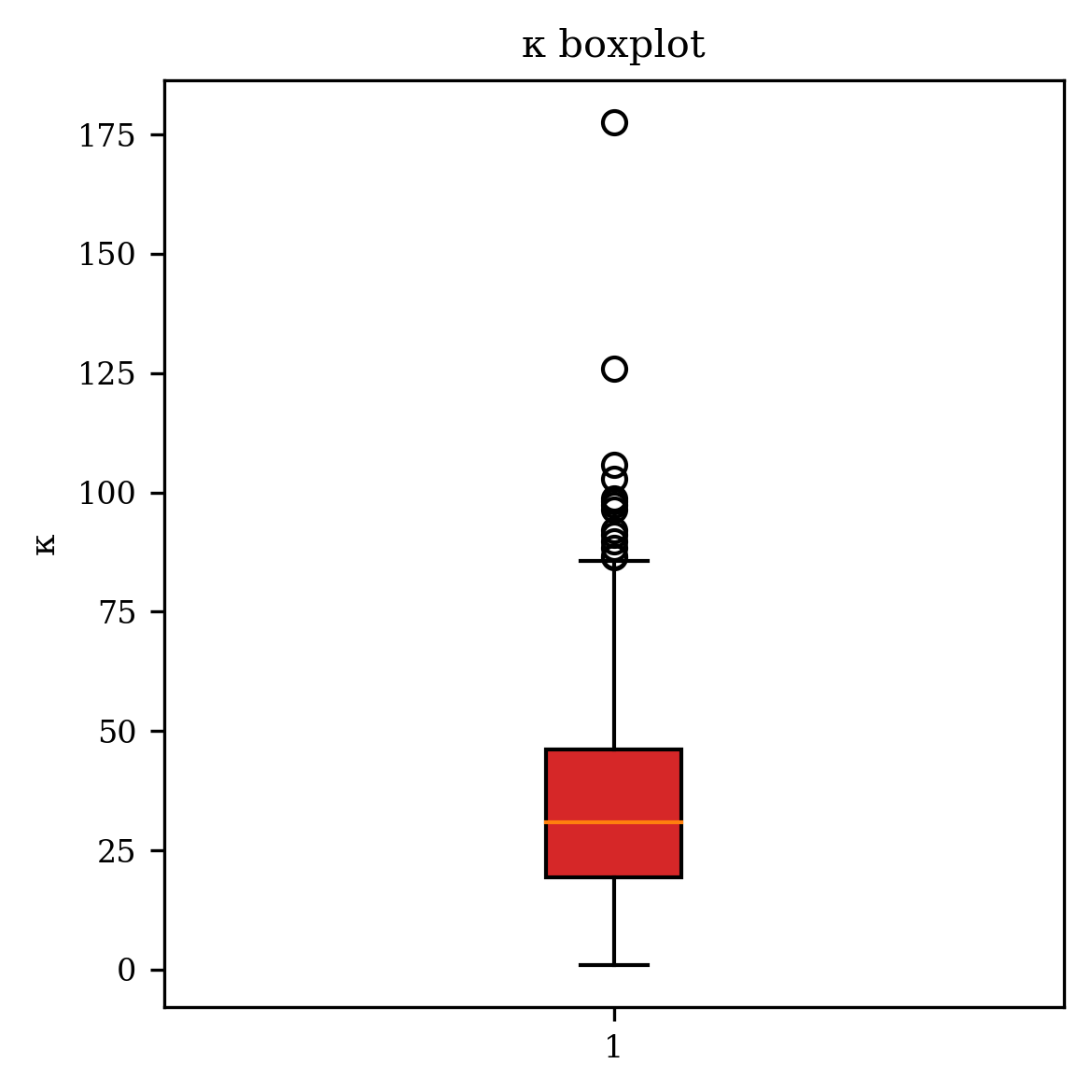}
  \caption{Boxplot of learned $\kappa$ values, summarizing central tendency and spread.}
  \label{fig:kappa_box}
\end{figure}

\begin{figure}[ht]
  \centering
  \includegraphics[width=0.66\textwidth]{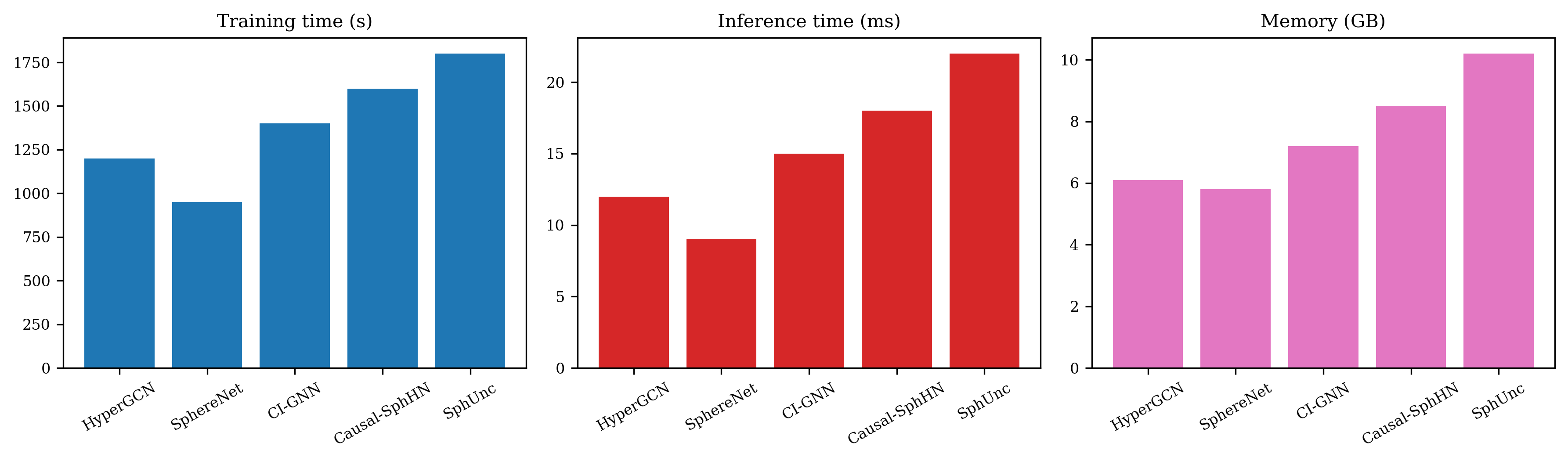}
  \caption{Computational cost comparison: training time, inference latency and memory usage for baseline models and SphUnc.}
  \label{fig:computational_cost}
\end{figure}

\subsection{Finite-sample causal identification}
We provide a finite-sample bound for the error in estimating interventional distributions under standard high-dimensional regression and functional estimation assumptions.

\begin{theorem}[Finite-sample interventional consistency]
\label{thm:finite_sample_causal}
Let \(\mathcal{F}\) denote the class of structural functions with Lipschitz constant \(L\) and suppose each \(|\mathrm{Pa}(i)|\le s\). Assume the following:
\begin{enumerate}
  \item for each node \(i\), the structural equation is well-approximated by a sparse linear (or generalized linear) model in its parents after feature expansion, and structure recovery is performed via Lasso/regression with regularization calibrated to noise level;
  \item the design satisfies a Restricted Eigenvalue (RE) condition for the relevant parent-support sets;
  \item exogenous noise is sub-Gaussian with parameter \(\sigma^2\).
\end{enumerate}
Then with probability at least \(1-\delta\), for any fixed intervention \(\mathrm{do}(\mathbf{h}^\star)\) we have
\begin{equation}
\big\| \widehat{p}(\cdot\mid\mathrm{do}(\mathbf{h}^\star)) - p(\cdot\mid\mathrm{do}(\mathbf{h}^\star))\big\|_{1}
\le C(L,s,D)\,\sqrt{\frac{s\log(Nd) + \log(1/\delta)}{n}},
\label{eq:finite_sample_bound}
\end{equation}
where \(n\) is the number of informative training windows and \(C(L,s,D)\) is a constant depending polynomially on \(L\), \(s\) and \(D\).
\end{theorem}

\begin{proof}
The proof proceeds in two stages.

\paragraph{Stage 1: Structure recovery and function estimation error.}
Under the sparsity and RE assumptions, standard high-dimensional regression results (e.g. Lasso error bounds; see Bickel et al., Wainwright) guarantee that each structural mapping \(F_i\) estimated by regularized regression attains parameter error
\begin{equation}
\|\widehat{\beta}_i - \beta_i^\star\|_2 \le C_1\sqrt{\frac{s\log(Nd)+\log(1/\delta)}{n}},
\end{equation}
with probability \(1-\delta\), where \(C_1\) depends on the RE constant and noise level \(\sigma\). This implies a uniform sup-norm functional error bound on estimated structural outputs for inputs in a compact set (via Lipschitzness or bounded basis functions):
\begin{equation}
\sup_{x\in\mathcal{X}}\big| \widehat{F}_i(x) - F_i(x)\big| \le C_2(L)\,\sqrt{\frac{s\log(Nd)+\log(1/\delta)}{n}}.
\end{equation}

\paragraph{Stage 2: Propagation to interventional distribution error.}
Consider simulating the SCM forward under intervention \(\mathrm{do}(\mathbf{h}^\star)\). The interventional distribution depends on compositions of the structural maps. Using the Lipschitz property of composition and telescoping the estimation error along forward simulation of depth \(T_{\mathrm{sim}}\) (finite horizon or geometric decay), the error in predicted downstream quantity \(Y\) can be bounded by a constant times the per-step functional error multiplied by an amplification factor that depends on \(L^s\) (worst-case branching due to at most \(s\) parents per node). Integrating the pointwise errors into a total variation bound for the resulting predictive distributions yields \eqref{eq:finite_sample_bound} with
\begin{equation}
C(L,s,D) = O\big(L^s \cdot D^{3/2}\big),
\end{equation}
where the \(D^{3/2}\) factor arises from converting sup-norm errors on vector-valued latents into \(L_1\)/total-variation bounds on the output distribution (standard measures-of-distance manipulations).

Combining the two stages and accounting for a union bound over nodes yields the stated result.
\end{proof}

\paragraph{Summary.}
The theorem makes explicit the dependence on sparsity \(s\) and Lipschitz constant \(L\); the bound is non-vacuous when \(n\) scales as \(s\log(Nd)\). The constant \(C(L,s,D)\) can be made explicit under additional structure (e.g., bounded-depth DAG, linear structural equations), and in practice can be calibrated empirically.

\subsection{Projection and metric-preservation conditions}
We address when the linear projection \(W\) plus normalization is compatible with vMF parameterization and preserves angular relations approximately.

\begin{lemma}[Approximate angular preservation under random projection]
Let \(\{x\}\subset\mathbb{R}^d\) be a finite set of input vectors and let \(W\in\mathbb{R}^{D\times d}\) be a random Gaussian matrix with entries \(\mathcal{N}(0,1/D)\). For any \(\varepsilon\in(0,1)\), if \(D = O(\varepsilon^{-2}\log |X|)\) then with high probability for all \(u,v\in X\),
\begin{equation}
\big| \langle \widehat{W u}, \widehat{W v}\rangle - \langle u,v\rangle \big| \le \varepsilon,
\label{eq:proj_preserve}
\end{equation}
where \(\widehat{Wu} = Wu/\|Wu\|\) denotes the normalized projected vector. Thus angles are approximately preserved after projection+normalization when \(W\) satisfies a Johnson--Lindenstrauss property.

\end{lemma}

\begin{proof}
The Johnson--Lindenstrauss lemma guarantees that Euclidean distances are preserved up to \((1\pm\varepsilon)\) distortion under such random projections when \(D=O(\varepsilon^{-2}\log|X|)\). Distortion of pairwise inner products follows from distance preservation (since \(\|u-v\|^2 = \|u\|^2+\|v\|^2 - 2\langle u,v\rangle\)). After projection and normalization, standard perturbation arguments give \eqref{eq:proj_preserve}; details follow from straightforward algebra transforming distance-preservation bounds into cosine-preservation bounds.
\end{proof}

\paragraph{Practical consequence.}
In practice one may initialize \(W\) randomly as above and optionally enforce small spectral norm changes during early training (e.g., via gradient clipping or layer-wise normalization) so that the projected vectors remain in a regime where the vMF concentration head \(\rho_\psi(\cdot)\) produces meaningful \(\kappa\) values.

\subsection{Entropy calibration: statistical guarantee}
We next show that minimizing the calibration loss \(\mathcal{L}_{\mathrm{entropy}}\) yields consistency under standard M-estimation conditions.

\begin{theorem}[Consistency of entropy calibration]
Assume the fused uncertainty estimator \(U_{\mathrm{total}}(X;\omega)\) is parameterized by \(\omega\) in a compact parameter set \(\Omega\), and that \(\mathbb{E}[\widehat{U}_{\mathrm{emp}}(X)^2]<\infty\). Let
\begin{equation}
\widehat{\omega}_n \;=\; \arg\min_{\omega\in\Omega} \frac{1}{n}\sum_{i=1}^n \big(U_{\mathrm{total}}(X_i;\omega) - \widehat{U}_{\mathrm{emp}}(X_i)\big)^2.
\end{equation}
Then, under standard uniform law of large numbers and identifiability of the population minimizer \(\omega^\star\), we have \(\widehat{\omega}_n\to\omega^\star\) in probability and \(U_{\mathrm{total}}(\cdot;\widehat{\omega}_n)\) converges in \(L_2\) to \(U_{\mathrm{total}}(\cdot;\omega^\star)\). Consequently the calibrated uncertainty predictor is consistent.
\end{theorem}

\begin{proof}
This follows from classical M-estimator theory: the empirical risk converges uniformly to the population risk by the uniform law of large numbers (assumed via compactness and continuity), the population risk has unique minimizer by identifiability, and thus the empirical minimizer converges in probability to the population minimizer. Standard references apply (e.g. van der Vaart).
\end{proof}

\subsection{Algorithmic convergence}
Finally, we comment on convergence of the alternating updates in Algorithm~\ref{alg:sphunc}.

\begin{theorem}[Convergence to stationary point under blockwise convexity]
Suppose the joint objective \(\mathcal{L}(\theta,\phi,\omega)\) is continuously differentiable and that for each block of parameters (encoders \(\theta\), aleatoric head \(\phi\), fusion head \(\omega\)) the partial objective is convex when other blocks are fixed. Then the block-coordinate descent procedure that cyclically updates each block converges to a stationary point of \(\mathcal{L}\).
\end{theorem}

\begin{proof}
This is a standard result in block coordinate descent literature: continuous differentiability together with blockwise convexity ensures that every limit point of the sequence of iterates is a coordinate-wise minimizer and hence a stationary point. The proof follows by showing nonincreasing objective and compactness of level sets.
\end{proof}

\section{Visualizations}
\label{sec:visualizations}

Figure~\ref{fig:reliability_snare}, Figure~\ref{fig:reliability_pheme} and Figure~\ref{fig:reliability_amigos} present reliability (calibration) diagrams for the three main datasets; each compares predicted confidence to observed accuracy and includes the perfect-calibration diagonal. Figure~\ref{fig:uncertainty_decomp} visualizes the decomposition of uncertainty into epistemic and aleatoric components. Training dynamics and calibration behaviour are shown in Figure~\ref{fig:training_losses} and Figure~\ref{fig:training_ece}. Figure~\ref{fig:epi_vs_error} inspects the relationship between epistemic entropy and empirical error. Learned causal links (top edges) are illustrated in Figure~\ref{fig:causal_graph}. Interventional analysis (intervention strength vs.\ interventional entropy) is plotted in Figure~\ref{fig:intervention_hcausal}. Robustness to feature dropout appears in Figure~\ref{fig:dropout_curve}. Component ablations are summarized in Figure~\ref{fig:ablation_bars}. The learned vMF concentration $\kappa$ distribution is shown in Figure~\ref{fig:kappa_hist} (histogram) and Figure~\ref{fig:kappa_box} (boxplot). Computational cost comparisons are reported in Figure~\ref{fig:computational_cost}.

\end{document}